\documentclass[screen]{acmart} %acmsmall

\usepackage{xcolor}
\newcommand{\colorline}[1]{{\color{black}{#1}}}

\newtheorem{theorem}{Theorem}
\newtheorem{definition}{Definition}
\newtheorem{assumption}{Assumption}

\AtBeginDocument{%
  }

\setcopyright{acmlicensed}
\acmJournal{TCPS}
\acmYear{2025} \acmVolume{1} \acmNumber{1} \acmArticle{1} \acmMonth{1}\acmDOI{10.1145/3778246}

\acmConference[Transactions on Cyber-Physical Systems]{Make sure to enter the correct
  conference title from your rights confirmation emai}%{June 03--05,2024}{Woodstock, NY}
\begin{document}

\title{LTD: Low Temperature Distillation for Gradient Masking-free Adversarial Training}

\author{Erh-Chung Chen}
\affiliation{%
  \institution{National Tsing Hua University}
  \city{HsinChu}
  \state{HsinChu}
  \country{Taiwan}}
\email{s107062802@m107.nthu.edu.tw}

\author{Che-Rung Lee}
\affiliation{%
  \institution{National Tsing Hua University}
  \city{HsinChu}
  \state{HsinChu}
  \country{Taiwan}}
\email{cherung@cs.nthu.edu.tw}

\renewcommand{\shortauthors}{Chen and Lee}

\begin{abstract}
    Adversarial training is a widely adopted strategy to bolster the robustness of neural network models against adversarial attacks. This paper revisits the fundamental assumptions underlying image classification and suggests that representing data as one-hot labels is a key factor that leads to vulnerabilities. However, in real-world datasets, data ambiguity often arises, with samples exhibiting characteristics of multiple classes, rendering one-hot label representations imprecise. To address this, we introduce a novel approach, Low-Temperature Distillation (LTD), designed to refine label representations. Unlike previous approaches, LTD incorporates a relatively low temperature in the teacher model, while maintaining a fixed temperature for the student model during both training and inference. This strategy not only refines assumptions about data distribution but also strengthens model robustness and avoids the gradient masking problem commonly encountered in defensive distillation. Experimental results demonstrate the efficacy of the proposed method when combined with existing frameworks, achieving robust accuracy rates of 58.19\%, 31.13\%, and 42.08\% on the CIFAR-10, CIFAR-100, and ImageNet datasets, respectively, without the need for additional data.
\end{abstract}

\begin{CCSXML}
<ccs2012>
<concept>
<concept_id>10010147.10010178</concept_id>
<concept_desc>Computing methodologies~Artificial intelligence</concept_desc>
<concept_significance>500</concept_significance>
</concept>
<concept>
<concept_id>10002978</concept_id>
<concept_desc>Security and privacy</concept_desc>
<concept_significance>500</concept_significance>
</concept>
<concept>
<concept_id>10010147.10010257</concept_id>
<concept_desc>Computing methodologies~Machine learning</concept_desc>
<concept_significance>500</concept_significance>
</concept>
</ccs2012>
\end{CCSXML}

\ccsdesc[500]{Computing methodologies~Artificial intelligence}
\ccsdesc[500]{Security and privacy}
\ccsdesc[500]{Computing methodologies~Machine learning}

\keywords{adversarial training, AI safety}


\maketitle

%%%%%%%%%%%%%%%%%%%%%%%%%%%%%%%%%%%%%%%%%%%%%%%%%%%%%%%%%%%%%%%%%%%%%
\section{Introduction}
\label{sec:intro}
Deep neural networks (DNNs) have achieved remarkable success across a wide range of complex tasks, including image classification \citep{krizhevsky2012imagenet}, object detection \citep{wang2022yolov7}, image captioning \citep{herdade2019image}, and semantic analysis \citep{zhang2018deep}. These capabilities underpin numerous real-world applications such as self-driving car \citep{grigorescu2020survey} and machine translation \citep{devlin2018bert}. As the deployment of DNNs expands across diverse and critical domains, there is growing attention toward practical challenges beyond mere accuracy. These include concerns such as model robustness, model compression, and low-precision training or inference.

Among these challenges, robustness has emerged as a particularly pressing concern. Despite their impressive performance, DNNs are vulnerable to adversarial attacks — strategically crafted perturbations that are often imperceptible to humans yet capable of misleading high-performing models \citep{kurakin2016adversariala, szegedy2013intriguing}. These adversarial examples are not confined to the digital domain; they have been shown to pose physical-world threats in applications such as the cell phone camera attack \citep{kurakin2018adversarial} or road sign attack \citep{eykholt2018robust,zolfi2021translucent}. Beyond inference-time attacks, backdoor attacks — manipulations introduced during the training process — further exacerbate the vulnerability of machine learning systems \citep{yao2019latent, chen2017targeted}.

The development of robust models capable of withstanding adversarial threats remains an ongoing and complex challenge. Among the most effective defense strategies is adversarial training, a method that generates adversarial examples during training and minimizes the objective caused by these adversarial examples. While adversarial training with PGD attack \cite{madry2019deep} and its varieties \cite{pmlr-v97-zhang19p,pang2020bag,gowal2020uncovering} have demonstrated promising results, the computational cost associated with adversarial training remains relatively high. To address this concern, faster training methods have been introduced to reduce training costs but accuracy is sacrificed \cite{shafahi2019adversarial,wong2020fast,Chen_2020_ACCV}.  Moreover, challenges such as gradient masking \citep{athalye2018obfuscated} — a phenomenon that obscures gradient information used in crafting adversarial examples — can lead to misleading estimates of robustness and complicate the evaluation of defense mechanisms.

A crucial but often overlooked aspect of model performance and robustness is label annotation. For image classification tasks, the quality and accuracy of label annotations directly affect the model's ability to learn effectively. Upon revisiting the foundational assumptions of image classification, we observe that none of those basic assumptions are satisfied in the real-world scenario. Typically, images are represented as one-hot vectors, where each image is assigned to exactly one class. However, this representation fails to capture the complex relationships between classes. In practice, real-world datasets often contain ambiguous samples, where images exhibit characteristics from multiple classes. This suggests that similar classes may share common features, leading to fuzzy semantic distances between them. Such oversimplified representations hinder effective learning and make models more susceptible to adversarial attacks. To address this limitation, it is crucial to develop more sophisticated label representations, and numerous strategies have been proposed, such as knowledge distillation \cite{hinton2015distilling}, label smoothing \cite{muller2019does}, learning from the noisy labels \cite{song2022learning}. One early approach, defensive distillation \cite{papernot2016distillation}, aimed to integrate knowledge distillation with adversarial training to enhance the robustness of the target models against adversarial examples. Unfortunately, several works have reported that this kind of strategy is unreliable \cite{athalye2018obfuscated,lee2020gradient}.

Motivated by these insights, this work introduces a novel knowledge distillation framework — Low-Temperature Distillation (LTD) — to enhance model robustness against adversarial attacks. In this paper, we focus on defending against restricted white-box attacks, where complete information about the target networks is accessible and perturbations generated by adversarial attacks are bounded. Our key contributions are summarized as follows:
\begin{itemize}
    \item We identify that conventional one-hot labeling exacerbates model vulnerability by failing to capture features from multiple classes. Our analyses reveal that even minimal label noise can significantly degrade model predictions and robustness.
    \item We propose Low-Temperature Distillation (LTD), a distillation-based framework that enables target models to learn richer inter-class feature representations from a teacher model. LTD applies fixed but distinct temperatures to the teacher and student models to better preserve informative gradients during training. 
    \item The role of the teacher model is to construct a better label representation. Nevertheless, robust teacher models are not an essential requirement for a knowledge distillation framework when applied to adversarial training.
    \item This modification enhances robustness without encountering the gradient masking problem during inference. Additionally, it preserves high-quality gradients for crafting adversarial examples during the training phase, particularly in the late stage of training.
    \item We conduct extensive evaluations on CIFAR-10, CIFAR-100, and ImageNet datasets. When integrated with Adversarial Weight Perturbation (AWP) \citep{wu2020adversarial}, LTD achieves robust accuracies of 58.19\% on CIFAR-10, 31.13\% on CIFAR-100, and 42.08\% on ImageNet — representing a notable improvement over prior methods under the same network architecture. 
\end{itemize}

The rest of this paper is organized as follows. Section \ref{sec:survey} reviews related work on adversarial training and knowledge distillation frameworks. Section \ref{sec:method} revisits assumptions of multi-class classification and motivates the use of soft labels for better robustness. Section \ref{sec:algo} details the proposed LTD algorithm. Experimental results are presented in Section \ref{sec:exp}. The conclusion and future work are given in the last section.

%%%%%%%%%%%%%%%%%%%%%%%%%%%%%%%%%%%%%%%%%%%%%%%%%%%%%%%%%%%%%%%%%%%%%
\section{Related Works}
\label{sec:survey}
In this section, we review prior work in two main areas: (1) adversarial attacks and defenses, and (2) the knowledge distillation framework.

%-------------------------------------------------------------------------
\subsection{Adversarial Attack}
Adversarial examples are inputs modified by imperceptible perturbations designed to deceive classifiers. While a model may correctly classify a clean input, its corresponding adversarial variant is often misclassified. Formally, adversarial examples are defined as:
\begin{equation}
\label{eq:adv_data}
\mathcal{S} = \left\{x' \middle\vert 
                    \begin{array}{l}
                         \arg \max Z(x;\theta)_i = \arg \max y_i \\ 
                         \arg \max Z(x;\theta)_i \neq \arg \max Z(x';\theta)_i  \\
                         ||x' - x ||_\infty \leq \epsilon
                    \end{array}
              \right\},
\end{equation}
where $x$ and $x'$ denote the original examples and the corresponding adversarial examples, respectively; the victim classifier $Z$ is a function of an input image and model's weights $\theta$, and $\epsilon$ is the allowed distance in $L_\infty$ space. The classifier outputs a logit $q_i$ for class $i$ and predicts the corresponding label by 
\begin{equation}
    h= \arg \max q_i.
\end{equation}

In practice, adversarial attacks can be classified into two categories based on the amount of information accessible to the attacker. In white-box attacks, the attacker has full information about the victim model including its architecture and parameters. These attacks typically leverage gradient information to generate adversarial examples by maximizing the loss function. Notable white-box techniques include Projected Gradient Descent (PGD) \citep{madry2019deep} and AutoAttack \citep{croce2020reliable}, which are widely adopted as benchmarks for evaluating robustness. In contrast, black-box attacks occur when the attacker has limited information about the victim model. In contrast, black-box attacks assume restricted access, typically limited to model outputs. Attackers estimate gradient directions by querying the model with multiple inputs and analyzing the response differences \citep{chen2017zoo, andriushchenko2019square}.

%-------------------------------------------------------------------------
\subsection{Adversarial Training}
Adversarial training is one of the most widely used techniques for improving model robustness. It involves generating adversarial examples during training and minimizing the loss of these examples. The standard adversarial training objective \citep{madry2019deep} is formulated as:
\begin{equation}
\label{eq:adv_train}
\min_{\theta} \mathop{\mathbb{E}}_{x \sim \mathcal{D}} \max_{\tilde{x} :D(\tilde{x} ,x) < \epsilon } {L(Z(\tilde{x}; \theta) , y)},
\end{equation}
% \mathop{\mathbb{E}}_{x \sim \mathcal{D}}
where $x$ is sampled from the data distribution $\mathcal{D}$, and $D(\tilde{x}, x)$ is a distance metric. The goal is to minimize the worst-case loss within the given $\epsilon$ ball. Since the strongest adversarial examples are indeterminable in advance, the inner maximization is usually replaced by known attacks in practice. However, a notable issue with this approach is that it can significantly degrade the model’s accuracy on original examples, as these are not sampled in the optimization process.

A significant limitation of this approach is that it often leads to reduced accuracy on clean data, as these samples are underrepresented in the optimization objective. To mitigate this issue, TRADES \citep{pmlr-v97-zhang19p} was introduced. TRADES decomposes the loss into two components: the standard classification loss on natural examples and a regularization term encouraging smooth decision boundaries:
\begin{equation}
\label{eq:TRADES}
    L_{\textrm{TRADES}}(x,x',y) = L(Z(x;\theta),y) + \lambda \Delta L(x,x',y;\theta),
\end{equation}
where $L(Z(x;\theta),y))$ is an ordinary objective and $\Delta L(x,x',y;\theta)$ is usually the KL divergence as a regularization term. The first term in (\ref{eq:TRADES}) maximizes the similarity between the output distribution and natural data with its corresponding label, while the second term encourages the output distribution to be smooth and pushes the decision boundary away from adversarial examples. The key benefit of KL divergence is that it is label-free, enabling the use of additional unlabeled data for improving robustness \cite{uesato2019labels,carmon2019unlabeled,gowal2020uncovering}.

Despite its effectiveness, adversarial training \cite{pmlr-v97-zhang19p,uesato2019labels,carmon2019unlabeled,gowal2020uncovering} is highly sensitive to the quality of gradients. If gradients are uninformative or obscured, the training process may fail to generate meaningful adversarial examples. This can result in an overestimation of robustness, especially against gradient-based attacks. Gradient masking refers to a phenomenon in which a model appears robust against gradient-based adversarial attacks but remains vulnerable to alternative attack strategies, such as gradient-free or black-box methods. A work \citep{athalye2018obfuscated} proposed a set of diagnostic criteria to identify this issue. Violations of these criteria often signal the presence of gradient masking and can lead to an overestimation of robustness. As a result, a model is considered truly robust only if it produces informative gradients and withstands both white-box and black-box attacks. Building on this, a guideline \citep{carlini2019evaluating} provided best practices for designing and evaluating adversarial defenses. These recommendations emphasize the importance of maintaining gradient quality, avoiding masking effects, and conducting comprehensive robustness assessments.

%-------------------------------------------------------------------------
\subsection{Knowledge Distillation}
\label{sec:KD}
Knowledge distillation \citep{hinton2015distilling} was initially proposed to compress large models for efficient deployment on resource-constrained devices. It has since evolved into a foundational component of numerous training strategies \citep{gou2020knowledge}. The central idea is to provide meaningful label representations from a pre-trained model, known as the teacher model, to guide the training of a student model. These soft labels contain richer class information than hard labels and can promote better generalization. Given a temperature parameter $\tau$, the softmax function used to generate soft labels $p$ is defined as:
\begin{equation}
\label{eqn:KD_softmax}
p^{T=\tau}_i = \frac{\exp(q/\tau)_i}{\sum^{k}_{j=1} \exp(q/\tau)_j},
\end{equation}
where $q$ is the logit computed by the given model. The target model is trained by the following loss function:
\begin{equation}
\label{eqn:target_loss}
    L = -\sum^{k}_{i=1}{y_i \log p_i^{s,T=1}} \\
    + \lambda \sum^{k}_{i=1}{p_i^{t,T=\tau} \log\left( \frac{p_i^{t,T=\tau}}{p_i^{s,T=\tau}}\right)} ,
%\begin{split}
%    L = &-\sum^{k}_{i=1}{y_i \log p_i^{s,T=1}} \\
%    &+ \lambda \sum^{k}_{i=1}{p_i^{t,T=\tau} \log\left( \frac{p_i^{t,T=\tau}}{p_i^{s,T=\tau}}\right)} ,
%\end{split}
\end{equation}
where $y$ is the given one-hot labels; $p^{s,T=\tau}$ and $p^{t,T=\tau}$ are labels given from the target model and the teacher model using the temperature $\tau$ respectively, and $\lambda$ is a scaled factor. The first term in (\ref{eqn:target_loss}) is an ordinary categorical loss and the second term is Kullback–Leibler Divergence (KLD). To maintain a balance between the two losses, the value of $\lambda$ is set to $\tau^2$ generally \cite{hinton2015distilling}.

Defensive distillation \citep{papernot2016distillation} was the first attempt to integrate this framework into adversarial defense. However, subsequent research identified key shortcomings, particularly the unintended effects of scaling via $\tau^2$, which can suppress gradient magnitudes and lead to gradient masking or vanishing. Recent work has addressed these concerns using gradient-free attack evaluations to assess actual robustness \citep{croce2020reliable}. This analytical method offers a clearer understanding of the limitations and challenges posed by unintended alterations in the defensive distillation process. 
%As the pursuit of robust models against adversarial threats continues, addressing and rectifying such issues becomes crucial for the advancement of adversarial defense mechanisms.

%%%%%%%%%%%%%%%%%%%%%%%%%%%%%%%%%%%%%%%%%%%%%%%%%%%%%%%%%%%%%%%%%%%%%
\section{Data Labeling}
\label{sec:method}
In this section, we revisit the fundamental assumptions of k-class classification in the ideal case, before exploring the challenges and limitations these assumptions encounter in real-world scenarios. We also propose an alternative label representation to more effectively address the presence of ambiguous data in real-world datasets
%-------------------------------------------------------------------------
\subsection{Fundamental Assumptions of Classification Problem}
\label{sec:def}
In a k-class classification problem, three implicit assumptions are required: the closed-world assumption, the independent and identically distributed (i.i.d.) assumption, and the clean and large data assumption \cite{zhang2020towards}. The closed-world assumption supposes that the number of the class $k$ is predefined and all examples must belong to exactly one predefined class. The i.i.d. assumption supposes that all samples on $\mathcal{D}_{\text{train}}$ and $\mathcal{D}_{\text{test}}$ are drawn from an identical distribution. Under the i.i.d. assumption, the objective can be approximated by empirical risk from observed samples on $\mathcal{D}_{\text{train}}$. The clean and big data assumption supposes that all collected data should be well-labeled and sufficiently large to cover the entire population.

With three assumptions, a model can effectively learn the distinct features of each class by minimizing an objective function iteratively across the training dataset. Under the closed-world assumption, the ground truth labels are typically represented as one-hot vectors. For a given label $y$, the one-hot vector $\mathbf{y}$ is defined as:
\begin{equation}
\label{eq:one_hot}
  \textbf{y} =
    \begin{cases}
      1 \quad & i = y, \\
      0 \quad & \text{otherwise}.
    \end{cases} 
\end{equation}
A softmax function, a special case of (\ref{eqn:KD_softmax}) with temperature $T=1$, is typically appended to the model to normalize the output and ensure it aligns with the probabilistic interpretation, where the sum of the probabilities for each class equals one. The optimization problem is generally expressed as:
\begin{equation}
\label{eq:nat_train}
\min \mathop{\mathbb{E}}_{x \sim \mathcal{D}} -\sum^{k}_{i=1}{y_i \log p_i},
\end{equation}
where $y_i$ and $p_i$ represent the ground truth label's probability and the model's predicted probability for the $i$-th class, respectively. In practice, the dataset is divided into two subsets: a training set and a validation set. The training set should be sufficiently large to cover all possible situations, with sample weights adjusted based on the frequency of their occurrence. The purpose of the validation set is to assess the model's ability to generalize, ensuring that the model performs well on unseen data.

%-------------------------------------------------------------------------
\subsection{Classification Problem in Real-world Scenario}
\label{sec:def_real}

\begin{figure}[t]
\centering
\begin{minipage}{.22\linewidth}
    \includegraphics[width=\textwidth]{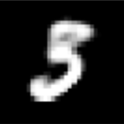}
\end{minipage}
\begin{minipage}{.22\linewidth}
    \includegraphics[width=\textwidth]{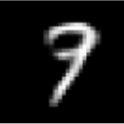}
\end{minipage}
\begin{minipage}{.22\linewidth}
    \includegraphics[width=\textwidth]{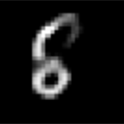}
\end{minipage}
\begin{minipage}{.22\linewidth}
    \includegraphics[width=\textwidth]{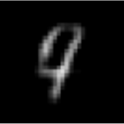}
\end{minipage}
\caption{Ambiguous images in MNIST dataset violate closed world assumption.}
\label{fig:ambiguous_mnist}
\end{figure}

\begin{figure}[t]
\centering
\begin{minipage}{.22\linewidth}
    \includegraphics[width=\textwidth]{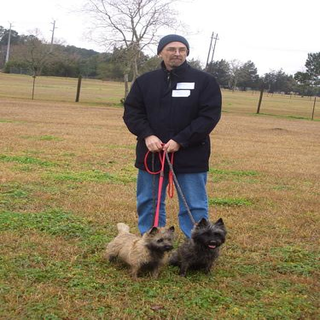}
\end{minipage}
\begin{minipage}{.22\linewidth}
    \includegraphics[width=\textwidth]{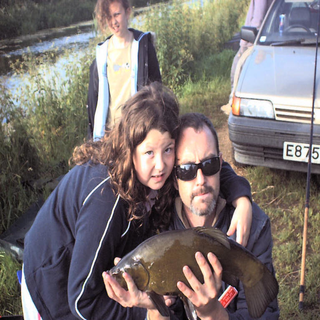}
\end{minipage}
\begin{minipage}{.22\linewidth}
    \includegraphics[width=\textwidth]{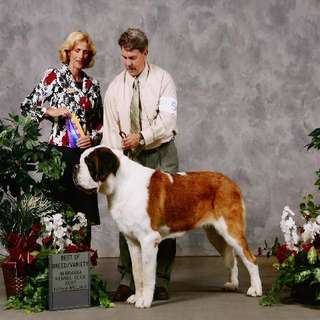}
\end{minipage}
\begin{minipage}{.22\linewidth}
    \includegraphics[width=\textwidth]{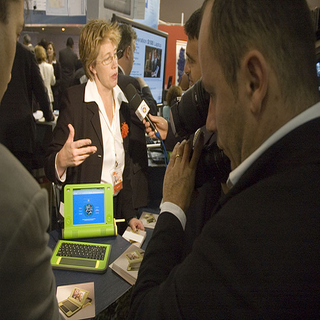}
\end{minipage}
\caption{Ambiguous images in ImageNet dataset contain multiple objects.}
\label{fig:imagenet_mult_objs}
\end{figure}

Although the original training procedure using softmax cross entropy loss (SCE) with one-hot vectors, as defined in (\ref{eq:nat_train}), achieves good performance with low empirical loss, we argue that this approach does not hold well in real-world scenarios. %The assumptions underlying this approach are rarely satisfied in practice. 
One of the key issues is that models trained with the standard SCE objective often struggle with samples drawn from diverse or mismatched distributions, leading to poor generalization. In these cases, the model exhibits high accuracy on the training data but fails to generalize effectively to unseen or adversarial examples. This phenomenon is typically referred to as overfitting \cite{rice2020overfitting}. As a result, several challenges arise when applying this framework to real-world data.

Moreover, minimizing the objective in (\ref{eq:nat_train}) often leads to overconfident predictions on ambiguous samples, with probabilities for non-selected classes fading to near zero. This is because the softmax function encourages the model to assign high probabilities to only one class, disregarding the uncertainty inherent in ambiguous data. From a sample-wise perspective, two images might belong to the same class but exhibit very different features. Figure \ref{fig:ambiguous_mnist} illustrates that certain images in MNIST dataset are located in the boundary of two classes in semantic space. For such samples, it is reasonable to consider them as belonging to more than one class. Similarly, Figure \ref{fig:imagenet_mult_objs} highlights images from the ImageNet dataset with multiple objects, further challenging the closed-world assumption.

\colorline{In addition to these concerns, the predefined classes often do not adequately account for the intra-class and inter-class relationships, resulting in non-uniform semantic distances between classes \cite{li2023modeling}.} For example, in CIFAR-10 dataset, the classes automobile and truck have a relatively short inter-class distance. Similarly, in the ImageNet dataset, the classes \emph{sunglass} (n04355933) and \emph{sunglasses} (n04356056) are nearly identical in meaning, but the current annotation system treats them as distinct classes. The same issue arises with \emph{laptop} (n03642806) and \emph{notebook} (n03832673), where the terms are often used interchangeably but \emph{notebook} may refer to a book of plain paper. Previous studies argued that the wrong annotation procedure might cause performance degradation \cite{beyer2020we,tsipras2020imagenet}. Recent studies suggest that re-labeling to accommodate multi-labels or utilizing hierarchical relationships, such as those defined in WordNet, can mitigate these issues \cite{yun2021re,beyer2020we,chatterjee2023imagenet}.

Based on these observations, \colorline{we argue that the objective defined in (\ref{eq:nat_train}) is not an adequate metric when ambiguous examples are present in the dataset. A situation may arise where two models exhibit the same accuracy, but one model misclassifies more trivial samples than the other. In this scenario, we would consider the latter model superior, despite its higher misclassification rate in the ambiguous area. However, the current objective function does not provide insight into how well the model handles such ambiguous samples.} %Specifically, it fails to account for intra-class distances, which are crucial for properly managing the overlap between classes in ambiguous cases. 
It is essential to design an alternative metric to measure the influence of ambiguous examples and precisely estimate distribution mismatch.

To address these shortcomings and break the closed-world assumption, we propose using oracle probability distributions rather than one-hot vectors for label encoding. We define the discrepancy between the model's output and the oracle distribution using the Kullback-Leibler Divergence (KLD), which can be formally expressed as:
\begin{definition}
Let $\mathcal{D}$ be the set of all data to be classified, $x$ be an instance in $\mathcal{D}$, $y^g(x)$ be the oracle probability distribution of $x$, and $p$ be the probability outputs by $Z$ of $x$. The discrepancy between the model's output and the oracle probability distribution is defined as:
\begin{equation}
\label{eq:prob_SCE}
    G(Z) =  \mathop{\mathbb{E}}_{(x,y^g) \sim \mathcal{D}} \left[ \sum^{k}_{i=1} {y^g_i \log \frac{y^g_i}{p_i}} \right]
    = \mathop{\mathbb{E}}_{(x,y^g) \sim \mathcal{D}} \left[ -\sum^{k}_{i=1} y^g_i \log p_i + \mathrm{Const.} \right].
%\begin{split}
%    G(Z) &=  \mathop{\mathbb{E}}_{(x,y^g) \sim \mathcal{D}} \left[ \sum^{k}_{i=1} {y^g_i \log \frac{y^g_i}{p_i}} \right] \\
%    &= \mathop{\mathbb{E}}_{(x,y^g) \sim \mathcal{D}} \left[ -\sum^{k}_{i=1} y^g_i \log p_i + \mathrm{Const.} \right].
%\end{split}
\end{equation}
\end{definition}
The constant term reflects the intrinsic properties of the entire dataset. In the case of an ideal dataset, this term vanishes, and the equation reduces to the standard softmax cross-entropy (SCE) loss. Once a dataset is provided, the value of this term becomes fixed and can be omitted without affecting the evaluation of the model's performance. In this framework, a lower value of $G(Z)$ indicates a better alignment between the model’s output and the oracle distribution, suggesting that the model has effectively captured the inherent uncertainty and ambiguity within the data.

%-------------------------------------------------------------------------
\subsection{Oracle Distribution Estimation}
\label{sec:soft_label}
The exploration of more effective label representations remains a crucial yet under-explored topic in machine learning. \colorline{In this paper, we propose the use of soft labels generated via knowledge distillation as a promising alternative to one-hot encoding. This approach aims to minimize the discrepancy defined in Equation \ref{eq:prob_SCE}, thereby better capturing the inherent ambiguity in real-world datasets.}

Knowledge distillation provides flexibility by allowing the adjustment of temperature, which can be tuned to reflect varying degrees of uncertainty across datasets, as illustrated in Figures~\ref{fig:ambiguous_mnist} and \ref{fig:imagenet_mult_objs}. Additionally, a pre-trained teacher model in this framework can capture more nuanced data distributions, including inter-class relationships, leading to more informative and expressive label representations. To support this approach, we conducted a theoretical analysis in a binary classification setting with a single example. The analysis shows that soft labels derived through low-temperature distillation can outperform one-hot labels under relatively mild conditions. Further details of this analysis are provided in Appendix \ref{sec:appendix_analysis}.

Intuitively, high-confidence examples exhibit a large tolerance for minor fluctuations in output probability, which do not significantly affect the final prediction. In contrast, fluctuations have a much greater impact on ambiguous examples. To investigate this, we performed simulations using synthesized data to mimic the distribution of real-world datasets, incorporating a proportion of ambiguous data. The oracle probabilities for this data were drawn from a Dirichlet distribution with varying concentration parameters, with the experimental configuration detailed in Appendix \ref{sec:appendix_simulation}. \colorline{Our results suggest that even a small proportion of noisy or ambiguous data can significantly degrade classifier performance. In contrast to one-hot encoding, which has limitations in representing ambiguity, the use of soft labels allows for a more flexible representation. By slightly increasing the temperature, we can adjust the label distribution without requiring additional data, thereby improving the generalization ability of the target model, even when the distribution is imprecise.} Consequently, we argue that soft labels can relax the implicit assumptions inherent in traditional classification problems and mitigate the influence of adversarial examples situated in ambiguous regions.

%%%%%%%%%%%%%%%%%%%%%%%%%%%%%%%%%%%%%%%%%%%%%%%%%%%%%%%%%%%%%%%%%%%%%
\section{Algorithm and Implementations}
\label{sec:algo}

%-------------------------------------------------------------------------
\subsection{Training Framework}
\label{sec:our_algo}
\begin{figure*}[t]
\centering
    \includegraphics[trim=2cm 3.5cm 2cm 0.5cm,clip=true, width=0.9\textwidth]{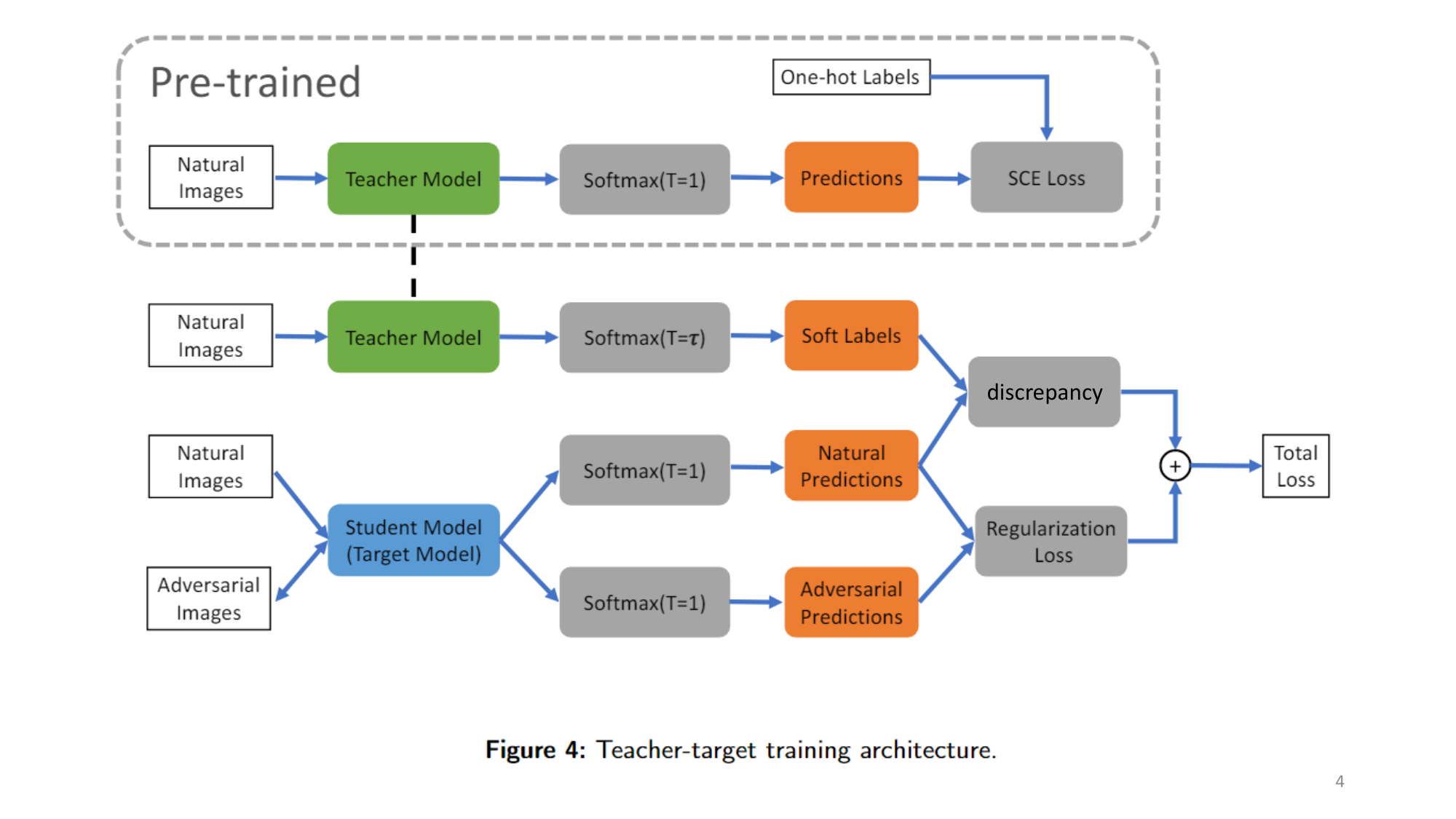}
    \caption{Teacher-target training architecture.}
\label{fig:arch}
\end{figure*}

Based on our analysis, we propose a training framework, called Low-Temperature Distillation (LTD), as illustrated in Figure \ref{fig:arch}. The framework follows the principles of knowledge distillation, where the teacher model in LTD is initially trained using the standard softmax cross-entropy (SCE) loss with natural images and one-hot encoded labels. The trained teacher model has high natural accuracy but is poor in robustness. The target model, which is the student model here, is trained by the same objective proposed by TRADES \cite{pmlr-v97-zhang19p} but replacing the one-hot label with the soft label generated by the teacher model. The soft labels can be generated in advance or on the fly. In contrast to defensive distillation \cite{papernot2016distillation}, LTD uses a relatively low temperature for the teacher model and employs distinct, fixed temperatures for the student model. During the training stage, adversarial examples are crafted using the target model's information and the original label $y$.

A fundamental distinction between our approach and TRADES, as well as related methods \cite{wu2020adversarial,pang2020bag}, lies in the distribution assumption for natural data. By replacing the one-hot labels with soft labels, we mitigate the risk of the output distribution converging to an incorrect distribution, thus avoiding an exacerbation of the divergence from the oracle distribution. Unlike previous approaches that often end training after the second learning rate decay to mitigate overfitting \cite{rice2020overfitting}, our approach involves mimicking the oracle distribution through a naturally trained model with low-temperature distillation.

We emphasize that this subtle but critical modification allows for a significant improvement in robustness while preserving natural accuracy, as it ensures the quality of gradients is preserved both during inference and throughout training, particularly in the later stages. Moreover, LTD is an easy-to-implement framework that can be integrated into existing models.

%-------------------------------------------------------------------------
\subsection{Quality of Gradients}
The quality of gradients is a critical factor in assessing whether the robustness of trained models has been overestimated. Previous studies have highlighted that defensive distillation suffers from the gradient masking problem \cite{athalye2018obfuscated}, as demonstrated by the Carlini-Wagner (CW) attack \cite{carlini2017towards}. The underlying cause of this issue lies in the manipulation of temperature during the distillation process, which alters the magnitude of gradients during both training and inference. Specifically, the gradient of the loss with respect to the input image, given the output probability, can be formulated as a function of the output probability:
\begin{equation}
\nabla_x L_{\mathrm{SCE}} = (p_t-1)\nabla_x q_t + \sum_{i \neq t}{p_i \nabla_x q_i},    
\end{equation}
where $p_i$ is the probability of class $i$ and $q_i$ is the logit of class $i$. This equation shows the gradient almost vanishes when $p_t$ is close to 1, making it challenging for attackers to to exploit gradient-based optimization.

In the defensive distillation framework, the temperatures of the target mode in the training stage ($T_t$) and in the inference stage ($T_i$) are different, where $T_t$ is high and $T_i=1$. It causes the magnitude of logits in the inference time is $(T_t/T_i)=T_t$ times larger than that of in training time. As a consequence, the largest logit dominates, and the output probability converges to a one-hot vector and falls into the area of gradient-vanishing. It leads to effectively masking the gradient and reducing the model's ability to respond to adversarial perturbations. 

In contrast, the proposed method maintains a consistent temperature for the target model throughout both the training and inference phases, ensuring that the magnitude of the gradient remains well-behaved. This prevents gradient masking and retains the model's vulnerability to small perturbations. As a result, our method avoids the gradient masking problem that typically arises in defensive distillation, thus enhancing the model's robustness.

Additionally, as the soft labels fed to the target model are generated through distillation at a relatively low temperature, this approach prevents the model from overfitting to the one-hot representations. For trivial examples, where the model can easily identify the correct class and the output probabilities of the predicted class are close to $1$, the gradient of the loss function tends to vanish, making gradient-based attacks ineffective. The use of soft labels slightly alters the value of $p_t$, ensuring the gradient does not converge to $1$ and the model retains a non-vanishing gradient, enabling the generation of high-quality adversarial examples for these trivial cases. This becomes particularly important in the middle and late stages of training when the model increasingly produces such examples.

In summary, our method employs fixed but distinct temperatures for teacher and student models, ensuring consistent logit scaling and preventing gradient masking. Additionally, the use of soft labels mitigates overfitting to hard targets and supports the generation of meaningful adversarial examples throughout training. This approach significantly enhances the robustness of the model while maintaining its performance.

%-------------------------------------------------------------------------
\subsection{Temperature Selection}
\label{sec:select_T}

\begin{figure*}[t]
\centering
    \includegraphics[width=0.95\textwidth]{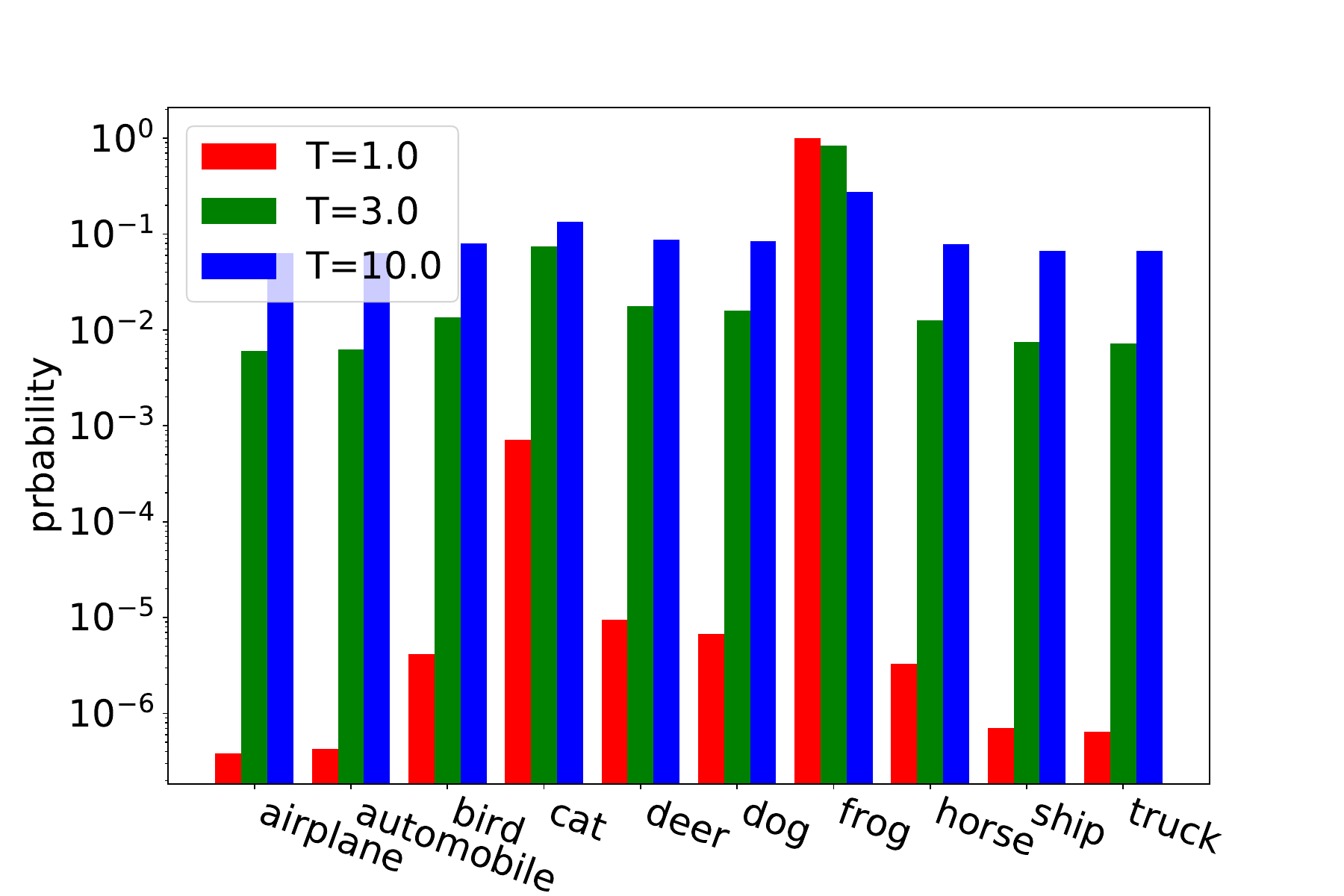}
    \caption{The output probability of the naturally trained model on different temperatures on CIFAR-10 dataset.}
\label{fig:temperature_prob_cifar10}
\end{figure*}

The optimal temperature for the teacher model depends on the dataset. For large-scale datasets like ImageNet, the output distribution may already be well-suited for LTD, and therefore, the temperature should be kept relatively low. In contrast, for datasets like CIFAR-10, which often output highly confident predictions, the output distribution closely resembles a one-hot encoding. In this case, it is beneficial to slightly increase the temperature to avoid overly confident predictions.

However, by the property of (\ref{eqn:KD_softmax}), if the temperature of the teacher model is excessively high, the probability of the winning class rapidly decreases and uniformly re-distributes the probability to the rest classes. As illustrated in Figure \ref{fig:temperature_prob_cifar10}, the probability distribution becomes more uniform as the temperature increases, undermining the model's ability to capture inter-class relationships. This uniformity causes the discrepancy defined in equation (\ref{eq:prob_SCE}) to be dominated by irrelevant classes, thus distorting the distribution and diminishing the relevance of robustness considerations.

To determine the optimal temperature for LTD training, we propose a two-step approach. In the first step, substitute models are trained solely on natural data, and we establish an eligible temperature range $[1, \tau_\mathrm{max}]$ based on whether the natural accuracy of the corresponding models surpasses a predefined threshold. In the second step, we determine the optimal temperature within this range by minimizing objective. It is important to note that the feasible temperature must be lower than $\tau_\mathrm{max}$ since adversarial examples hurt the natural accuracy. By searching for the optimal temperature in this two-step approach, we can ensure that the target model is trained with a temperature that is appropriate for the given data set and model architecture, resulting in improved robustness and natural accuracy.

%%%%%%%%%%%%%%%%%%%%%%%%%%%%%%%%%%%%%%%%%%%%%%%%%%%%%%%%%%%%%%%%%%%%%
\section{Experiments}
\label{sec:exp}
This section presents the robustness of LTD against white-box attacks, comparisons with other methods on CIFAR-10; CIFAR-100, and ImageNet datasets, and ablation studies in the temperature and $\lambda$. \colorline{Due to the high computational cost of full adversarial training, we were unable to explore the full range of training configurations and hyperparameter variations for transformer architecture. Despite this limitation, we have listed the robustness results in Appendix \ref{sec:vit_exp}.}

\subsection{Experiment Configurations}
\label{sec:appendix_config}
We implemented our algorithm using CUDA 11.3, cuDNN 8.4 and PyTorch 1.11.0 framework with mixed precision on NVIDIA 32GB V100. %All experiments are tested 5 times.

\subsubsection{Dataset}
\label{sec:exp:dataset}
CIFAR-10 and CIFAR-100 datasets contain 50,000 training images and 10,000 testing images. Each image is of dimension 32x32 and belongs to a specific class. 
%We divide 500 images as the validation set from training images. 
The values of all images value are normalized in range in $[0,1]$. Each image in training set are cropped randomly with padding 4 pixels on each border and applied random horizontal flip during training procedure.

ImageNet consists of about 1,300,00 training images collected from real world and each image's  dimensions are vary. We apply random resized crop with 224x224 during the training phase and center crop with 224x224 in inference phase. Since we assume that adversarial examples and natural images may have different statistical information, the commonly used normalization by subtracting the mean and dividing by standard deviation may not be suitable for this case. Instead, each pixel is normalized in range in $[0, 1]$.

\subsubsection{Implementations}
Our implementation builds upon the TRADES \cite{pmlr-v97-zhang19p} and AWP \cite{wu2020adversarial} frameworks to demonstrate that LTD enhances model robustness. For the CIFAR-10 and CIFAR-100 datasets, we use the Wide Residual Network (WRN) \cite{zagoruyko2016wide} with depth 34 and width 10. For ImageNet, we use ResNet-50 and WRN-50-2 architectures. The teacher and target models share identical architectures and hyperparameter configurations. Teacher models are trained from scratch using standard training with one-hot labels, and data augmentation follows the procedure described in Section \ref{sec:exp:dataset}. The teacher models achieve a natural accuracy of at least 94.5\% on CIFAR-10 and 77\% on CIFAR-100. For ImageNet, we use pre-trained models provided by the official PyTorch repository.

For CIFAR-10 and CIFAR-100, target models are trained for 120 epochs. The initial learning rate is set to $0.1$ and reduced by a factor of 10 at the 80\textsuperscript{th} and 100\textsuperscript{th} epochs. We use stochastic gradient descent (SGD) with a momentum of $0.9$, a weight decay of $5e^{-4}$, and Nesterov acceleration enabled. All other settings follow the defaults of the original implementations. Adversarial examples are generated using PGD-8 for both TRADES and AWP.

For the ImageNet experiments, training proceeds for $120$ epochs. The initial learning rate is $0.1$, decayed by a factor of 10 at the 50\textsuperscript{th} and 90\textsuperscript{th} epochs. The optimizer is SGD with a momentum of $0.9$, a weight decay of $1e^{-4}$, Nesterov momentum enabled, and a batch size of 320. Adversarial examples are generated using PGD-6.

\subsubsection{Metric}
The robust accuracy is measured under the white-box environment against the AutoAttack (AA) \cite{croce2020reliable} with default configuration. AutoAttack, includes APGD-CE, APGD-T, FAB-T, and Square attack. APGD-CE, APGD-T, and FAB-T are gradient-based attacks with different objectives or updating rules. The square attack is a query-efficient black-box attack that can detect the gradient masking effect. To improve computational efficiency, AutoAttack first filters out examples already misclassified by the current attack. Only the remaining samples are passed to the subsequent attacks. The final robust accuracy reported by AA accounts for the most successful adversarial attempt among all four attacks, thus ensuring that robustness is not overestimated. For CIFAR-10 and CIFAR-100, the baseline is $\epsilon = 8 / 255$ in $L_{\infty}$ norm. For ImageNet, Robustbench \cite{croce2020robustbench} evaluates the robustness using fixed 5,000 images from validation set within $\epsilon = 4 / 255$ in $L_{\infty}$ norm.

%-------------------------------------------------------------------------
\begin{table}[t]
\begin{minipage}{0.44\textwidth}
%\begin{table}[t]
\centering
\caption{Competitors from Robustbench on ImageNet \cite{croce2020robustbench}}
\label{table:baseline_imagenet}
\begin{tabular}{ r|c c c c}
    \# & paper & architecture           & acc\textsubscript{nat}[\%] & acc\textsubscript{AA}[\%]  \\
     \hline
     * & LTD                            & WRN-50-2 & 68.10 & 42.08 \\  %report/feature/adv_imagenet/0779d6d-0
     \hline
     1 & \cite{debenedetti2023light}    & XCiT-S12  & 72.34 & 41.78 \\
     2 & \cite{mo2022adversarial}       & Swin-B    & 74.66 & 38.30 \\
     3 & \cite{salman2020adversarially} & WRN-50-2  & 68.46 & 38.14 \\
     \hline
     * & LTD                            & ResNet-50 & 62.40 & 36.82 \\  %report/feature/adv_imagenet/0779d6d-0
     \hline
     4 & \cite{salman2020adversarially} & ResNet-50 & 64.02 & 34.96 \\
     5 & \cite{mo2022adversarial}       & ViT-B     & 68.38 & 34.40 \\
     6 & \cite{robustness}              & ResNet-50 & 62.56 & 29.22 \\
     7 & \cite{wong2020fast}            & ResNet-50 & 55.62 & 26.24 \\
     8 & \cite{salman2020adversarially} & ResNet-18 & 52.92 & 25.32 \\
\end{tabular}
\end{minipage}
\quad
\begin{minipage}{0.52\textwidth}
\centering
\caption{Competitors from Robustbench on CIFAR-10 \cite{croce2020robustbench}}
\label{table:baseline}
\begin{tabular}{ r|c c c c}
    \# & paper & architecture           & acc\textsubscript{nat}[\%] & acc\textsubscript{AA}[\%]  \\
    \hline
     * & AWP \cite{wu2020adversarial} + LTD & WRN-34-20 & 86.28 & 58.19 \\
     \hline
     1 & \cite{addepalli2021towards}    & WRN-34-10 & 85.32 & 58.04  \\
     2 & \cite{gowal2020uncovering}     & WRN-70-16 & 85.29 & 57.20 \\
     \hline
     * & AWP + LTD                       & WRN-34-10 & 85.21 & 56.90 \\
     \hline
     3 & \cite{gowal2020uncovering}     & WRN-34-20 & 85.64 & 56.86 \\
     4 & AWP                            & WRN-34-10 & 85.36 & 56.17 \\
     \hline
     * & TRADES +  LTD                   & WRN-34-10 & 85.63 & 55.09 \\
     \hline
     5 & \cite{pang2020bag}             & WRN-34-20 & 86.43 & 54.39 \\
     6 & \cite{pang2020bag}             & WRN-34-10 & 85.49 & 53.94 \\
     7 & \cite{pang2020boosting}        & WRN-34-20 & 85.14 & 53.74 \\
     8 & \cite{cui2020learnable}        & WRN-34-20 & 88.70 & 53.57 \\
     9 & \cite{zhang2020attacks}        & WRB-34-10 & 84.52 & 53.51 \\
    \hline 
    - & TRADES \cite{pmlr-v97-zhang19p} & WRN-34-10 & 84.92 & 53.08
\end{tabular}
%\end{table}
\end{minipage}
\end{table}

\begin{table}[tb]
\centering
\caption{Competitors from Robustbench on CIFAR-100 \cite{croce2020robustbench}}
\label{table:baseline_cifar100}
\begin{tabular}{ r|c c c c}
    \# & paper & architecture           & acc\textsubscript{nat}[\%] & acc\textsubscript{AA}[\%]  \\
    \hline
     * & AWP \cite{wu2020adversarial} + LTD & WRN-34-10 & 64.32 & 31.13 \\
     \hline
     1 & \cite{cui2020learnable}        & WRN-34-20 & 62.55 & 30.20 \\
     2 & \cite{gowal2020uncovering}     & WRN-70-16 & 60.86 & 30.03 \\
     3 & AWP \cite{wu2020adversarial}   & WRN-34-10 & 60.38 & 28.86
\end{tabular}
\end{table}

%-------------------------------------------------------------------------
\subsection{White-box Robustness}
Table \ref{table:baseline_imagenet}, \ref{table:baseline}, and \ref{table:baseline_cifar100}  present the experimental results of ImageNet, CIFAR-10, and CIFAR-100, respectively. In these tables, acc\textsubscript{nat} represents the accuracy on natural data, and acc\textsubscript{AA} represents the robust accuracy against AA attack. The orders of the methods are based on their robustness accuracy acc\textsubscript{AA}. The numbers in $\#$ denote the original rankings of other methods in RobustBench \cite{croce2020robustbench}, and the items with $*$ denote our results. To make a fair comparison, the methods involving generative examples or data from external datasets are excluded. We also included the results of TRADES at the bottom row as a baseline for comparison.

As can be seen, For CIFAR-10 dataset, integrating our method LTD with TRADES increases robust accuracy from 53.08\% to 55.09\%. Further combining LTD with AWP improves it from 56.17\% to 56.90\% using WRN-34-10. When using WRN-34-20, AWP+LTD achieves a robust accuracy of 58.19\%, the highest among all evaluated methods. For CIFAR-100, combining AWP with LTD improves robust accuracy from 28.86\% to 31.13\% with WRN-34-10—again, the best result among methods that do not rely on additional data, while also maintaining a smaller model size compared to other approaches. Similarly, on the ImageNet dataset, LTD improves robustness from 38.14\% to 42.08\%, outperforming several transformer-based models.

These results demonstrate the effectiveness of LTD in enhancing the robustness of deep learning models, particularly on challenging datasets like ImageNet that contain ambiguous and multi-object images. The improvements align with our assumptions outlined in Section \ref{sec:method}. Moreover, LTD preserves gradient quality, ensuring robustness is not overestimated due to gradient masking.

%-------------------------------------------------------------------------
\subsection{Optimal Temperature Selection}
We employed WRN-34-10 architecture with the CIFAR-10 dataset, training the model using TRADES+LTD with varying temperatures in the teacher model to investigate the significance of temperature selection. As mentioned in Section \ref{sec:select_T}, the best temperature is selected in two steps.

As shown in Table \ref{table:temperature_test}, we searched for the best temperature in the range of $[1.0, 50.0]$, excluding the upper end since its natural accuracy is 86.63\% which is too low to accept. The experimental results presented in Table \ref{table:temperature_test} demonstrate that the best temperature for the teacher model is $5.0$. As the temperature increases, the robust accuracy decreases due to the fact that irrelevant classes receive partial probability from the target class, which violates our assumption. This leads to gradient masking, and we can observe that the robust accuracy in the training phase is almost 100\%, but it cannot defend against stronger attacks or unseen attacks. AA identifies the occurrence of gradient masking when the temperature is higher than $15.0$. This phenomenon also has been observed in defensive distillation, where the magnitude of logit is significantly altered.

It is important to emphasize that our method outperforms TRADES across various temperatures. This indicates that the one-hot assumption may be overly confident in real-world datasets. There are several distribution assumptions that can achieve the same or above performance with the existing criteria. Instead, soft label representations generated by the teacher models with low temperatures are good choices although the used temperature is not optimal. Moreover, the above phenomenon is consistent with the simulated results.

\colorline{In practice, selecting the best temperature depends on resource constraints. If the budget allows, comprehensive experiments can be conducted to identify the optimal temperature. However, based on the empirical results of this experiment, we recommend selecting a temperature within the range of $[1, 5]$. Setting the temperature to $1$ as a baseline and gradually increasing it will allow for assessing the cost-effectiveness of investing more resources in finding a better temperature.}

\begin{table}[t]
\centering
\caption{Ablation study on temperature selection.}
\label{table:temperature_test}
\begin{tabular}{ r|c c}
    \ & acc\textsubscript{nat}[\%] & acc\textsubscript{AA}[\%] \\
    \hline
    TRADES & 84.92 & 53.08 \\
    \hline
    T=1.0  & 84.51 & 54.38 \\
    T=2.0  & 84.96 & 54.90 \\
    T=3.0  & 85.48 & 55.03\\
    T=5.0  & \textbf{86.20} & \textbf{55.09}\\
    T=8.0  & 85.23 & 54.63 \\
    T=10.0 & 84.72 & 53.56 \\
    T=12.0 & 77.45 & 43.85 \\
    \hline
    T=15.0 & 94.72 & 0.00 \\
    T=20.0 & 94.39 & 0.00 \\
    \hline
    T=50.0 & 86.63 & * \\
\end{tabular}
\end{table}

%----------------------------------------------------------------
\subsection{Ablation study}
\label{sec:appendix_ablation_lambda}

\begin{table}[t]
\begin{minipage}{0.45\textwidth}
\centering
\caption{Ablation study on CIFAR-10}
\label{table:ablation_lambda_cifar10}
\begin{tabular}{ c |c c c}
    $\lambda$ & architecture  & acc\textsubscript{nat}[\%] & acc\textsubscript{AA}[\%]  \\
     \hline
         3.0  & WRN-34-10 & \textbf{87.66} & 56.39 \\
         4.0  & WRN-34-10 & 86.61 & 56.61 \\
         5.0  & WRN-34-10 & 85.81 & 56.72 \\
         6.0  & WRN-34-10 & 85.15 & \textbf{56.91} \\
         7.0  & WRN-34-10 & 84.97 & 56.20 \\
\end{tabular}
\end{minipage}
\quad
\begin{minipage}{0.45\textwidth}
\centering
\caption{Ablation study on CIFAR-100}
\label{table:ablation_lambda_cifar100}
\begin{tabular}{ c |c c c}
    $\lambda$ & architecture  & acc\textsubscript{nat}[\%] & acc\textsubscript{AA}[\%]  \\
     \hline
         3.0  & WRN-34-10 & \textbf{66.67} & 30.63 \\
         4.0  & WRN-34-10 & 65.16 & 30.83 \\
         5.0  & WRN-34-10 & 63.85 & 31.05 \\
         6.0  & WRN-34-10 & 63.31 & \textbf{31.13} \\
         7.0  & WRN-34-10 & 63.06 & 30.94\\
\end{tabular}
\end{minipage}
\end{table}

The primary purpose of $\lambda$ is to strike a balance between the two objectives. Ideally, $\lambda$ should be instance-dependent, and its value can be increased to enhance the robustness of well-classified examples. However, previous works suggested that a fixed value can achieve significant success \cite{wu2020adversarial, gowal2020uncovering}, with the value typically set to 6. Nevertheless, the impact of the hyper-parameter $\lambda$ under a soft label scenario remains unknown.

This experiment investigates the influence of the choice of various $\lambda$ for natural accuracy and robust accuracy. We conducted the ablation study with LTD+AWP on CIFAR-10 and CIFAR-100 datasets. Table \ref{table:ablation_lambda_cifar10} and Table \ref{table:ablation_lambda_cifar100} show the results on CIFAR-10 and CIFAR-100 datasets, respectively. As can be seen, the natural accuracy decreases when the value of $\lambda$ increases while the lowest natural accuracy occurs when $\lambda$ is $7.0$ on both data sets. These results suggest the best choice of $\lambda$ must be lower than $7.0$. Meanwhile, the best robust accuracy occurs when $\lambda$ is $6.0$ which is consistent with previous works. However, if we prefer higher natural accuracy, the optimal configuration of $\lambda$ can be a smaller value.

%----------------------------------------------------------------
\subsection{Discussion}

The proposed training framework, which replaces the one-hot labels with soft labels, achieves better robustness than the original training framework, even though the oracle probability distribution remains unknown. We emphasize that \textbf{the proposed method is the first method that effectively mitigates the gradient masking issue that has been observed in defensive distillation}. This approach is suitable for the classification problem which violates closed-world, or clean and big data assumptions. An excellent example of such a problem is ImageNet, where our approach demonstrated a significant improvement of about four percent in terms of robustness. However, it is essential to recognize and address the potential implications linked to our work, as discussed below.

\textbf{Knowledge Distillation} In this paper, we revisited knowledge distillation frameworks for adversarial training and provided a practical solution to correct the gradient masking issue. These results bring a chance that a robust but lightweight or compressed model can be obtained by using knowledge distillation frameworks. However, the robustness evaluation is more sophisticated for compressed models, since the target models may be defeated by adversarial examples crafted by unpruned models.

We acknowledge that the selection of the optimal temperature for knowledge distillation depends on the complexity of the target dataset. In cases where examples can be classified with high accuracy and minimal ambiguity, there is no necessity to modify their label representations. However, for ambiguous examples, it becomes crucial to dynamically adjust the temperature based on the model's confidence in the classification. In our work, we employed a global temperature across all examples, though we recognize that this approach may benefit from refinement in future studies. A potential solution to this issue is to use an auxiliary classifier to identify ambiguous examples and use the output information as a metric for designing temperature adjustment strategies.

Additionally, the choice of the teacher model is an open problem. We chose to train the teacher model naturally, rather than adversarially, because adversarial training may not necessarily produce good teachers in our framework. Adversarial training tends to prioritize robustness over accuracy, and as a result, the natural accuracy of the teacher model may be relatively low. We experimented with using an adversarially trained model as the teacher but found that the trained student models achieved lower natural accuracy. For the CIFAR-10 dataset, the natural accuracy of the student model was lower than 82\%, which is generally unacceptable. While adversarially trained models may extract underlying features that can improve robustness, several issues need to be overcome before they can be used effectively as teacher models in our framework.

On the other hand, recent studies have proposed an alternative implementation of knowledge distillation, involving a teacher model that can defend against adversarial attacks \cite{zi2021revisiting,zhu2021reliable}. Nevertheless, a key distinction between these works and our own lies in the foundational assumption pertaining to label representations. We posit that the susceptibility of DNNs stems from flawed data labeling. Low-temperature distillation stands as one among various methods capable of implicitly recalibrating data distribution and thereby enhancing robustness. Remarkably, the teacher models in our approach remain unburdened by any constraints, demonstrating substantial robustness. Our framework streamlines the incorporation of knowledge distillation into adversarial training, presenting a more straightforward integration design.

\textbf{Label Representation} Estimating the oracle distribution is a crucial problem, but LTD is not the only solution to predict the oracle distribution. AVmixup \cite{lee2020adversarial} augmented training examples by defining virtual examples, utilizing linear interpolation for both input examples and their labels. While the model trained by AVmixup can mitigate the overfitting issue, it cannot defend against CW attacks or PGD attacks with different objectives, especially for CIFAR-100 and more sophisticated data sets. On the other hand, CCAT \cite{stutz2020confidence} further suggested that the label representations should be calibrated by the strength of adversarial examples. Another approach suggested by \cite{chen2021robust} is to train the model from two teachers, one being a naturally trained model and the other being an adversarially trained model, which can alleviate overfitting. Ensemble training may also be a proper solution to avoid the overconfidence problem, but there is a lack of systematic studies for designing teacher models, including the total number of teacher models, the training policy, and the used architectures for each teacher model.

\textbf{Adversarial Training with Additional Data} The most effective strategy for improving the robustness is using additional examples from external data sets. The major reason is that traditional supervised learning almost ignores low-frequency data so DNNs cannot recognize them correctly. Adding extra data can cover the data in the regime of low frequency in the original training set. UAT \cite{uesato2019labels} showed TRADES on CIFAR-10 with 200,000 additional images improved the robustness significantly. In UAT, the authors also concluded that selecting a subset from the additional images properly has better robustness than that using entire additional images. The follow-up work RST \cite{carmon2019unlabeled} designed a special classifier to select relevant images, which joined into the adversarial training set, from another dataset.

However, this type of approach is excluded from our approach. We believe that selecting data under the same distribution is an opening problem, particularly when considering that the distribution of jointly trained examples undergoes a complete shift if chosen by an auxiliary classifier in advance. Furthermore, another criticism is that the number of additional images is much larger than that of CIFAR-10's training set (50,000 images), and processing those data requires more computational cost. Those similar works, which rely on huge additional images, cannot be extended to ImageNet or other large-scale data sets. In contrast, LTD can achieve significant improvement in the robustness of ImageNet data set. However, we believe that one direction of future work is to identify outliers in advance and exclude those examples from the training set to accelerate training speed. 

Our experimental results indicate that increasing the temperature of the teacher model can effectively mimic soft labels and enhance robustness without the need for additional data. While finding the optimal temperature requires training multiple substitute models, which can be costly, even setting the temperature to $1$ yields better robustness than the baseline. Overall, the only additional expense for our approach is training the teacher model with clean data, which is relatively minimal.

%%%%%%%%%%%%%%%%%%%%%%%%%%%%%%%%%%%%%%%%%%%%%%%%%%%%%%%%%%%%%%%%%%%%%
\section{Conclusion}
The use of one-hot labels as a representation for classification problems was common practice. However, in this paper, we have demonstrated that one-hot labels are imprecise and one of the vulnerabilities of DNNs arises from ambiguous examples. We have also explored the advantages of soft labels, which can significantly decrease discrepancies between the oracle probability and predictions from a model trained by soft labels. To train robust DNNs, we proposed a modified knowledge distillation framework that utilizes soft labels with a properly low temperature. By using different temperatures for the teacher and student models, we avoided the issue of gradient masking, while the soft labels were able to represent smoother probability distributions among classes. The experimental results on CIFAR-10 and CIFAR-100 data sets show that our approach, when combined with AWP, achieves robust accuracy of 58.19\% and 31.13\% respectively. For ImageNet dataset, our approach gains about 4\% improvement on the robustness. These results provide insights into how to label annotations affect the robustness of deep neural networks. We believe that designing a better label representation for real-world scenarios remains an unexplored issue.

\begin{acks}
We thank the founding support from National Science and Technology Council  ( 113-2622-E-007 -018 ) and the computing infrastructure provided by the National Center for High-Performance Computing, National Institutes of Applied Research (NIAR), Taiwan.
\end{acks}

\bibliographystyle{ACM-Reference-Format}
\bibliography{sample-manuscript}


\begin{thebibliography}{60}

%%% the Web address.

\ifx \showCODEN    \undefined \def \showCODEN     #1{\unskip}     \fi
\ifx \showDOI      \undefined \def \showDOI       #1{#1}\fi
\ifx \showISBNx    \undefined \def \showISBNx     #1{\unskip}     \fi
\ifx \showISBNxiii \undefined \def \showISBNxiii  #1{\unskip}     \fi
\ifx \showISSN     \undefined \def \showISSN      #1{\unskip}     \fi
\ifx \showLCCN     \undefined \def \showLCCN      #1{\unskip}     \fi
\ifx \shownote     \undefined \def \shownote      #1{#1}          \fi
\ifx \showarticletitle \undefined \def \showarticletitle #1{#1}   \fi
\ifx \showURL      \undefined \def \showURL       {\relax}        \fi
% invisible to TeX
\providecommand\bibfield[2]{#2}
\providecommand\bibinfo[2]{#2}
\providecommand\natexlab[1]{#1}
\providecommand\showeprint[2][]{arXiv:#2}

\bibitem[Addepalli et~al\mbox{.}(2021)]%
        {addepalli2021towards}
\bibfield{author}{\bibinfo{person}{Sravanti Addepalli}, \bibinfo{person}{Samyak Jain}, \bibinfo{person}{Gaurang Sriramanan}, \bibinfo{person}{Shivangi Khare}, {and} \bibinfo{person}{Venkatesh~Babu Radhakrishnan}.} \bibinfo{year}{2021}\natexlab{}.
\newblock \showarticletitle{Towards Achieving Adversarial Robustness Beyond Perceptual Limits}. In \bibinfo{booktitle}{\emph{ICML 2021 Workshop on Adversarial Machine Learning}}.
\newblock
\urldef\tempurl%
\url{https://openreview.net/forum?id=SHB_znlW5G7}
\showURL{%
\tempurl}


\bibitem[Alayrac et~al\mbox{.}(2019)]%
        {uesato2019labels}
\bibfield{author}{\bibinfo{person}{Jean-Baptiste Alayrac}, \bibinfo{person}{Jonathan Uesato}, \bibinfo{person}{Po-Sen Huang}, \bibinfo{person}{Alhussein Fawzi}, \bibinfo{person}{Robert Stanforth}, {and} \bibinfo{person}{Pushmeet Kohli}.} \bibinfo{year}{2019}\natexlab{}.
\newblock \showarticletitle{Are labels required for improving adversarial robustness?}
\newblock \bibinfo{journal}{\emph{Advances in Neural Information Processing Systems}}  \bibinfo{volume}{32} (\bibinfo{year}{2019}).
\newblock


\bibitem[Ali et~al\mbox{.}(2021)]%
        {ali2021xcit}
\bibfield{author}{\bibinfo{person}{Alaaeldin Ali}, \bibinfo{person}{Hugo Touvron}, \bibinfo{person}{Mathilde Caron}, \bibinfo{person}{Piotr Bojanowski}, \bibinfo{person}{Matthijs Douze}, \bibinfo{person}{Armand Joulin}, \bibinfo{person}{Ivan Laptev}, \bibinfo{person}{Natalia Neverova}, \bibinfo{person}{Gabriel Synnaeve}, \bibinfo{person}{Jakob Verbeek}, {et~al\mbox{.}}} \bibinfo{year}{2021}\natexlab{}.
\newblock \showarticletitle{Xcit: Cross-covariance image transformers}.
\newblock \bibinfo{journal}{\emph{Advances in neural information processing systems}}  \bibinfo{volume}{34} (\bibinfo{year}{2021}), \bibinfo{pages}{20014--20027}.
\newblock


\bibitem[Andriushchenko et~al\mbox{.}(2020)]%
        {andriushchenko2019square}
\bibfield{author}{\bibinfo{person}{Maksym Andriushchenko}, \bibinfo{person}{Francesco Croce}, \bibinfo{person}{Nicolas Flammarion}, {and} \bibinfo{person}{Matthias Hein}.} \bibinfo{year}{2020}\natexlab{}.
\newblock \showarticletitle{Square attack: a query-efficient black-box adversarial attack via random search}. In \bibinfo{booktitle}{\emph{European conference on computer vision}}. Springer, \bibinfo{pages}{484--501}.
\newblock


\bibitem[Athalye et~al\mbox{.}(2018)]%
        {athalye2018obfuscated}
\bibfield{author}{\bibinfo{person}{Anish Athalye}, \bibinfo{person}{Nicholas Carlini}, {and} \bibinfo{person}{David Wagner}.} \bibinfo{year}{2018}\natexlab{}.
\newblock \showarticletitle{Obfuscated gradients give a false sense of security: Circumventing defenses to adversarial examples}. In \bibinfo{booktitle}{\emph{International Conference on Machine Learning}}. PMLR, \bibinfo{pages}{274--283}.
\newblock


\bibitem[Beyer et~al\mbox{.}(2020)]%
        {beyer2020we}
\bibfield{author}{\bibinfo{person}{Lucas Beyer}, \bibinfo{person}{Olivier~J H{\'e}naff}, \bibinfo{person}{Alexander Kolesnikov}, \bibinfo{person}{Xiaohua Zhai}, {and} \bibinfo{person}{A{\"a}ron van~den Oord}.} \bibinfo{year}{2020}\natexlab{}.
\newblock \showarticletitle{Are we done with imagenet?}
\newblock \bibinfo{journal}{\emph{arXiv preprint arXiv:2006.07159}} (\bibinfo{year}{2020}).
\newblock


\bibitem[Carlini et~al\mbox{.}(2019)]%
        {carlini2019evaluating}
\bibfield{author}{\bibinfo{person}{Nicholas Carlini}, \bibinfo{person}{Anish Athalye}, \bibinfo{person}{Nicolas Papernot}, \bibinfo{person}{Wieland Brendel}, \bibinfo{person}{Jonas Rauber}, \bibinfo{person}{Dimitris Tsipras}, \bibinfo{person}{Ian Goodfellow}, \bibinfo{person}{Aleksander Madry}, {and} \bibinfo{person}{Alexey Kurakin}.} \bibinfo{year}{2019}\natexlab{}.
\newblock \showarticletitle{On evaluating adversarial robustness}.
\newblock \bibinfo{journal}{\emph{arXiv preprint arXiv:1902.06705}} (\bibinfo{year}{2019}).
\newblock


\bibitem[Carlini and Wagner(2017)]%
        {carlini2017towards}
\bibfield{author}{\bibinfo{person}{Nicholas Carlini} {and} \bibinfo{person}{David Wagner}.} \bibinfo{year}{2017}\natexlab{}.
\newblock \showarticletitle{Towards evaluating the robustness of neural networks}. In \bibinfo{booktitle}{\emph{2017 ieee symposium on security and privacy (sp)}}. IEEE, \bibinfo{pages}{39--57}.
\newblock


\bibitem[Carmon et~al\mbox{.}(2019)]%
        {carmon2019unlabeled}
\bibfield{author}{\bibinfo{person}{Yair Carmon}, \bibinfo{person}{Aditi Raghunathan}, \bibinfo{person}{Ludwig Schmidt}, \bibinfo{person}{John~C Duchi}, {and} \bibinfo{person}{Percy~S Liang}.} \bibinfo{year}{2019}\natexlab{}.
\newblock \showarticletitle{Unlabeled data improves adversarial robustness}.
\newblock \bibinfo{journal}{\emph{Advances in Neural Information Processing Systems}}  \bibinfo{volume}{32} (\bibinfo{year}{2019}).
\newblock


\bibitem[Chatterjee et~al\mbox{.}(2023)]%
        {chatterjee2023imagenet}
\bibfield{author}{\bibinfo{person}{Ankita Chatterjee}, \bibinfo{person}{Jayanta Mukherjee}, {and} \bibinfo{person}{Partha~Pratim Das}.} \bibinfo{year}{2023}\natexlab{}.
\newblock \showarticletitle{ImageNet classification using wordnet hierarchy}.
\newblock \bibinfo{journal}{\emph{IEEE Transactions on Artificial Intelligence}} \bibinfo{volume}{5}, \bibinfo{number}{4} (\bibinfo{year}{2023}), \bibinfo{pages}{1718--1727}.
\newblock


\bibitem[Chen and Lee(2020)]%
        {Chen_2020_ACCV}
\bibfield{author}{\bibinfo{person}{Erh-Chung Chen} {and} \bibinfo{person}{Che-Rung Lee}.} \bibinfo{year}{2020}\natexlab{}.
\newblock \showarticletitle{Towards Fast and Robust Adversarial Training for Image Classification}. In \bibinfo{booktitle}{\emph{Proceedings of the Asian Conference on Computer Vision}}.
\newblock


\bibitem[Chen et~al\mbox{.}(2017b)]%
        {chen2017zoo}
\bibfield{author}{\bibinfo{person}{Pin-Yu Chen}, \bibinfo{person}{Huan Zhang}, \bibinfo{person}{Yash Sharma}, \bibinfo{person}{Jinfeng Yi}, {and} \bibinfo{person}{Cho-Jui Hsieh}.} \bibinfo{year}{2017}\natexlab{b}.
\newblock \showarticletitle{Zoo: Zeroth order optimization based black-box attacks to deep neural networks without training substitute models}. In \bibinfo{booktitle}{\emph{Proceedings of the 10th ACM Workshop on Artificial Intelligence and Security}}. \bibinfo{pages}{15--26}.
\newblock


\bibitem[Chen et~al\mbox{.}(2021)]%
        {chen2021robust}
\bibfield{author}{\bibinfo{person}{Tianlong Chen}, \bibinfo{person}{Zhenyu Zhang}, \bibinfo{person}{Sijia Liu}, \bibinfo{person}{Shiyu Chang}, {and} \bibinfo{person}{Zhangyang Wang}.} \bibinfo{year}{2021}\natexlab{}.
\newblock \showarticletitle{Robust overfitting may be mitigated by properly learned smoothening}. In \bibinfo{booktitle}{\emph{International Conference on Learning Representations}}.
\newblock


\bibitem[Chen et~al\mbox{.}(2017a)]%
        {chen2017targeted}
\bibfield{author}{\bibinfo{person}{Xinyun Chen}, \bibinfo{person}{Chang Liu}, \bibinfo{person}{Bo Li}, \bibinfo{person}{Kimberly Lu}, {and} \bibinfo{person}{Dawn Song}.} \bibinfo{year}{2017}\natexlab{a}.
\newblock \showarticletitle{Targeted backdoor attacks on deep learning systems using data poisoning}.
\newblock \bibinfo{journal}{\emph{arXiv preprint arXiv:1712.05526}} (\bibinfo{year}{2017}).
\newblock


\bibitem[Croce et~al\mbox{.}(2021)]%
        {croce2020robustbench}
\bibfield{author}{\bibinfo{person}{Francesco Croce}, \bibinfo{person}{Maksym Andriushchenko}, \bibinfo{person}{Vikash Sehwag}, \bibinfo{person}{Edoardo Debenedetti}, \bibinfo{person}{Nicolas Flammarion}, \bibinfo{person}{Mung Chiang}, \bibinfo{person}{Prateek Mittal}, {and} \bibinfo{person}{Matthias Hein}.} \bibinfo{year}{2021}\natexlab{}.
\newblock \showarticletitle{RobustBench: a standardized adversarial robustness benchmark}. In \bibinfo{booktitle}{\emph{Thirty-fifth Conference on Neural Information Processing Systems Datasets and Benchmarks Track (Round 2)}}.
\newblock
\urldef\tempurl%
\url{https://openreview.net/forum?id=SSKZPJCt7B}
\showURL{%
\tempurl}


\bibitem[Croce and Hein(2020)]%
        {croce2020reliable}
\bibfield{author}{\bibinfo{person}{Francesco Croce} {and} \bibinfo{person}{Matthias Hein}.} \bibinfo{year}{2020}\natexlab{}.
\newblock \showarticletitle{Reliable evaluation of adversarial robustness with an ensemble of diverse parameter-free attacks}. In \bibinfo{booktitle}{\emph{International conference on machine learning}}. PMLR, \bibinfo{pages}{2206--2216}.
\newblock


\bibitem[Cui et~al\mbox{.}(2021)]%
        {cui2020learnable}
\bibfield{author}{\bibinfo{person}{Jiequan Cui}, \bibinfo{person}{Shu Liu}, \bibinfo{person}{Liwei Wang}, {and} \bibinfo{person}{Jiaya Jia}.} \bibinfo{year}{2021}\natexlab{}.
\newblock \showarticletitle{Learnable boundary guided adversarial training}. In \bibinfo{booktitle}{\emph{Proceedings of the IEEE/CVF International Conference on Computer Vision}}. \bibinfo{pages}{15721--15730}.
\newblock


\bibitem[Debenedetti et~al\mbox{.}(2023)]%
        {debenedetti2023light}
\bibfield{author}{\bibinfo{person}{Edoardo Debenedetti}, \bibinfo{person}{Vikash Sehwag}, {and} \bibinfo{person}{Prateek Mittal}.} \bibinfo{year}{2023}\natexlab{}.
\newblock \showarticletitle{A light recipe to train robust vision transformers}. In \bibinfo{booktitle}{\emph{2023 IEEE Conference on Secure and Trustworthy Machine Learning (SaTML)}}. IEEE, \bibinfo{pages}{225--253}.
\newblock


\bibitem[Devlin et~al\mbox{.}(2019)]%
        {devlin2018bert}
\bibfield{author}{\bibinfo{person}{Jacob Devlin}, \bibinfo{person}{Ming-Wei Chang}, \bibinfo{person}{Kenton Lee}, {and} \bibinfo{person}{Kristina Toutanova}.} \bibinfo{year}{2019}\natexlab{}.
\newblock \showarticletitle{Bert: Pre-training of deep bidirectional transformers for language understanding}. In \bibinfo{booktitle}{\emph{Proceedings of the 2019 conference of the North American chapter of the association for computational linguistics: human language technologies, volume 1 (long and short papers)}}. \bibinfo{pages}{4171--4186}.
\newblock


\bibitem[Engstrom et~al\mbox{.}(2019)]%
        {robustness}
\bibfield{author}{\bibinfo{person}{Logan Engstrom}, \bibinfo{person}{Andrew Ilyas}, \bibinfo{person}{Hadi Salman}, \bibinfo{person}{Shibani Santurkar}, {and} \bibinfo{person}{Dimitris Tsipras}.} \bibinfo{year}{2019}\natexlab{}.
\newblock \bibinfo{title}{Robustness (Python Library)}.
\newblock
\newblock
\urldef\tempurl%
\url{https://github.com/MadryLab/robustness}
\showURL{%
\tempurl}


\bibitem[Eykholt et~al\mbox{.}(2018)]%
        {eykholt2018robust}
\bibfield{author}{\bibinfo{person}{Kevin Eykholt}, \bibinfo{person}{Ivan Evtimov}, \bibinfo{person}{Earlence Fernandes}, \bibinfo{person}{Bo Li}, \bibinfo{person}{Amir Rahmati}, \bibinfo{person}{Chaowei Xiao}, \bibinfo{person}{Atul Prakash}, \bibinfo{person}{Tadayoshi Kohno}, {and} \bibinfo{person}{Dawn Song}.} \bibinfo{year}{2018}\natexlab{}.
\newblock \showarticletitle{Robust physical-world attacks on deep learning visual classification}. In \bibinfo{booktitle}{\emph{Proceedings of the IEEE Conference on Computer Vision and Pattern Recognition}}. \bibinfo{pages}{1625--1634}.
\newblock


\bibitem[Gou et~al\mbox{.}(2020)]%
        {gou2020knowledge}
\bibfield{author}{\bibinfo{person}{Jianping Gou}, \bibinfo{person}{Baosheng Yu}, \bibinfo{person}{Stephen~John Maybank}, {and} \bibinfo{person}{Dacheng Tao}.} \bibinfo{year}{2020}\natexlab{}.
\newblock \showarticletitle{Knowledge distillation: A survey}.
\newblock \bibinfo{journal}{\emph{arXiv preprint arXiv:2006.05525}} (\bibinfo{year}{2020}).
\newblock


\bibitem[Gowal et~al\mbox{.}(2020)]%
        {gowal2020uncovering}
\bibfield{author}{\bibinfo{person}{Sven Gowal}, \bibinfo{person}{Chongli Qin}, \bibinfo{person}{Jonathan Uesato}, \bibinfo{person}{Timothy Mann}, {and} \bibinfo{person}{Pushmeet Kohli}.} \bibinfo{year}{2020}\natexlab{}.
\newblock \showarticletitle{Uncovering the Limits of Adversarial Training against Norm-Bounded Adversarial Examples}.
\newblock \bibinfo{journal}{\emph{arXiv preprint arXiv:2010.03593}} (\bibinfo{year}{2020}).
\newblock


\bibitem[Grigorescu et~al\mbox{.}(2020)]%
        {grigorescu2020survey}
\bibfield{author}{\bibinfo{person}{Sorin Grigorescu}, \bibinfo{person}{Bogdan Trasnea}, \bibinfo{person}{Tiberiu Cocias}, {and} \bibinfo{person}{Gigel Macesanu}.} \bibinfo{year}{2020}\natexlab{}.
\newblock \showarticletitle{A survey of deep learning techniques for autonomous driving}.
\newblock \bibinfo{journal}{\emph{Journal of Field Robotics}} \bibinfo{volume}{37}, \bibinfo{number}{3} (\bibinfo{year}{2020}), \bibinfo{pages}{362--386}.
\newblock


\bibitem[Herdade et~al\mbox{.}(2019)]%
        {herdade2019image}
\bibfield{author}{\bibinfo{person}{Simao Herdade}, \bibinfo{person}{Armin Kappeler}, \bibinfo{person}{Kofi Boakye}, {and} \bibinfo{person}{Joao Soares}.} \bibinfo{year}{2019}\natexlab{}.
\newblock \showarticletitle{Image captioning: Transforming objects into words}.
\newblock \bibinfo{journal}{\emph{Advances in neural information processing systems}}  \bibinfo{volume}{32} (\bibinfo{year}{2019}).
\newblock


\bibitem[Hinton et~al\mbox{.}(2015)]%
        {hinton2015distilling}
\bibfield{author}{\bibinfo{person}{Geoffrey Hinton}, \bibinfo{person}{Oriol Vinyals}, {and} \bibinfo{person}{Jeffrey Dean}.} \bibinfo{year}{2015}\natexlab{}.
\newblock \showarticletitle{Distilling the Knowledge in a Neural Network}. In \bibinfo{booktitle}{\emph{NIPS Deep Learning and Representation Learning Workshop}}.
\newblock
\urldef\tempurl%
\url{http://arxiv.org/abs/1503.02531}
\showURL{%
\tempurl}


\bibitem[Krizhevsky et~al\mbox{.}(2012)]%
        {krizhevsky2012imagenet}
\bibfield{author}{\bibinfo{person}{Alex Krizhevsky}, \bibinfo{person}{Ilya Sutskever}, {and} \bibinfo{person}{Geoffrey~E Hinton}.} \bibinfo{year}{2012}\natexlab{}.
\newblock \showarticletitle{Imagenet classification with deep convolutional neural networks}.
\newblock \bibinfo{journal}{\emph{Advances in neural information processing systems}}  \bibinfo{volume}{25} (\bibinfo{year}{2012}).
\newblock


\bibitem[Kurakin et~al\mbox{.}(2017)]%
        {kurakin2016adversariala}
\bibfield{author}{\bibinfo{person}{Alexey Kurakin}, \bibinfo{person}{Ian~J. Goodfellow}, {and} \bibinfo{person}{Samy Bengio}.} \bibinfo{year}{2017}\natexlab{}.
\newblock \showarticletitle{Adversarial Machine Learning at Scale}. In \bibinfo{booktitle}{\emph{5th International Conference on Learning Representations, {ICLR} 2017, Toulon, France, April 24-26, 2017, Conference Track Proceedings}}. \bibinfo{publisher}{OpenReview.net}.
\newblock
\urldef\tempurl%
\url{https://openreview.net/forum?id=BJm4T4Kgx}
\showURL{%
\tempurl}


\bibitem[Kurakin et~al\mbox{.}(2018)]%
        {kurakin2018adversarial}
\bibfield{author}{\bibinfo{person}{Alexey Kurakin}, \bibinfo{person}{Ian~J Goodfellow}, {and} \bibinfo{person}{Samy Bengio}.} \bibinfo{year}{2018}\natexlab{}.
\newblock \showarticletitle{Adversarial examples in the physical world}.
\newblock In \bibinfo{booktitle}{\emph{Artificial intelligence safety and security}}. \bibinfo{publisher}{Chapman and Hall/CRC}, \bibinfo{pages}{99--112}.
\newblock


\bibitem[Lee et~al\mbox{.}(2020a)]%
        {lee2020gradient}
\bibfield{author}{\bibinfo{person}{Hyungyu Lee}, \bibinfo{person}{Ho Bae}, {and} \bibinfo{person}{Sungroh Yoon}.} \bibinfo{year}{2020}\natexlab{a}.
\newblock \showarticletitle{Gradient masking of label smoothing in adversarial robustness}.
\newblock \bibinfo{journal}{\emph{IEEE Access}}  \bibinfo{volume}{9} (\bibinfo{year}{2020}), \bibinfo{pages}{6453--6464}.
\newblock


\bibitem[Lee et~al\mbox{.}(2020b)]%
        {lee2020adversarial}
\bibfield{author}{\bibinfo{person}{Saehyung Lee}, \bibinfo{person}{Hyungyu Lee}, {and} \bibinfo{person}{Sungroh Yoon}.} \bibinfo{year}{2020}\natexlab{b}.
\newblock \showarticletitle{Adversarial vertex mixup: Toward better adversarially robust generalization}. In \bibinfo{booktitle}{\emph{Proceedings of the IEEE/CVF Conference on Computer Vision and Pattern Recognition}}. \bibinfo{pages}{272--281}.
\newblock


\bibitem[Li et~al\mbox{.}(2023)]%
        {li2023modeling}
\bibfield{author}{\bibinfo{person}{Wenbin Li}, \bibinfo{person}{Zhichen Fan}, \bibinfo{person}{Jing Huo}, {and} \bibinfo{person}{Yang Gao}.} \bibinfo{year}{2023}\natexlab{}.
\newblock \showarticletitle{Modeling inter-class and intra-class constraints in novel class discovery}. In \bibinfo{booktitle}{\emph{Proceedings of the IEEE/CVF conference on computer vision and pattern recognition}}. \bibinfo{pages}{3449--3458}.
\newblock


\bibitem[Liu et~al\mbox{.}(2021)]%
        {liu2021swin}
\bibfield{author}{\bibinfo{person}{Ze Liu}, \bibinfo{person}{Yutong Lin}, \bibinfo{person}{Yue Cao}, \bibinfo{person}{Han Hu}, \bibinfo{person}{Yixuan Wei}, \bibinfo{person}{Zheng Zhang}, \bibinfo{person}{Stephen Lin}, {and} \bibinfo{person}{Baining Guo}.} \bibinfo{year}{2021}\natexlab{}.
\newblock \showarticletitle{Swin transformer: Hierarchical vision transformer using shifted windows}. In \bibinfo{booktitle}{\emph{Proceedings of the IEEE/CVF international conference on computer vision}}. \bibinfo{pages}{10012--10022}.
\newblock


\bibitem[Madry et~al\mbox{.}(2018)]%
        {madry2019deep}
\bibfield{author}{\bibinfo{person}{Aleksander Madry}, \bibinfo{person}{Aleksandar Makelov}, \bibinfo{person}{Ludwig Schmidt}, \bibinfo{person}{Dimitris Tsipras}, {and} \bibinfo{person}{Adrian Vladu}.} \bibinfo{year}{2018}\natexlab{}.
\newblock \showarticletitle{Towards Deep Learning Models Resistant to Adversarial Attacks}. In \bibinfo{booktitle}{\emph{International Conference on Learning Representations}}.
\newblock
\urldef\tempurl%
\url{https://openreview.net/forum?id=rJzIBfZAb}
\showURL{%
\tempurl}


\bibitem[Mo et~al\mbox{.}(2022)]%
        {mo2022adversarial}
\bibfield{author}{\bibinfo{person}{Yichuan Mo}, \bibinfo{person}{Dongxian Wu}, \bibinfo{person}{Yifei Wang}, \bibinfo{person}{Yiwen Guo}, {and} \bibinfo{person}{Yisen Wang}.} \bibinfo{year}{2022}\natexlab{}.
\newblock \showarticletitle{When adversarial training meets vision transformers: Recipes from training to architecture}.
\newblock \bibinfo{journal}{\emph{Advances in Neural Information Processing Systems}}  \bibinfo{volume}{35} (\bibinfo{year}{2022}), \bibinfo{pages}{18599--18611}.
\newblock


\bibitem[M{\"u}ller et~al\mbox{.}(2019)]%
        {muller2019does}
\bibfield{author}{\bibinfo{person}{Rafael M{\"u}ller}, \bibinfo{person}{Simon Kornblith}, {and} \bibinfo{person}{Geoffrey~E Hinton}.} \bibinfo{year}{2019}\natexlab{}.
\newblock \showarticletitle{When does label smoothing help?}
\newblock \bibinfo{journal}{\emph{Advances in neural information processing systems}}  \bibinfo{volume}{32} (\bibinfo{year}{2019}).
\newblock


\bibitem[Pang et~al\mbox{.}(2021)]%
        {pang2020bag}
\bibfield{author}{\bibinfo{person}{Tianyu Pang}, \bibinfo{person}{Xiao Yang}, \bibinfo{person}{Yinpeng Dong}, \bibinfo{person}{Hang Su}, {and} \bibinfo{person}{Jun Zhu}.} \bibinfo{year}{2021}\natexlab{}.
\newblock \showarticletitle{Bag of Tricks for Adversarial Training}. In \bibinfo{booktitle}{\emph{International Conference on Learning Representations}}.
\newblock
\urldef\tempurl%
\url{https://openreview.net/forum?id=Xb8xvrtB8Ce}
\showURL{%
\tempurl}


\bibitem[Pang et~al\mbox{.}(2020)]%
        {pang2020boosting}
\bibfield{author}{\bibinfo{person}{Tianyu Pang}, \bibinfo{person}{Xiao Yang}, \bibinfo{person}{Yinpeng Dong}, \bibinfo{person}{Kun Xu}, \bibinfo{person}{Jun Zhu}, {and} \bibinfo{person}{Hang Su}.} \bibinfo{year}{2020}\natexlab{}.
\newblock \showarticletitle{Boosting adversarial training with hypersphere embedding}.
\newblock \bibinfo{journal}{\emph{Advances in Neural Information Processing Systems}}  \bibinfo{volume}{33} (\bibinfo{year}{2020}), \bibinfo{pages}{7779--7792}.
\newblock


\bibitem[Papernot et~al\mbox{.}(2016)]%
        {papernot2016distillation}
\bibfield{author}{\bibinfo{person}{Nicolas Papernot}, \bibinfo{person}{Patrick McDaniel}, \bibinfo{person}{Xi Wu}, \bibinfo{person}{Somesh Jha}, {and} \bibinfo{person}{Ananthram Swami}.} \bibinfo{year}{2016}\natexlab{}.
\newblock \showarticletitle{Distillation as a defense to adversarial perturbations against deep neural networks}. In \bibinfo{booktitle}{\emph{2016 IEEE symposium on security and privacy (SP)}}. IEEE, \bibinfo{pages}{582--597}.
\newblock


\bibitem[Rice et~al\mbox{.}(2020)]%
        {rice2020overfitting}
\bibfield{author}{\bibinfo{person}{Leslie Rice}, \bibinfo{person}{Eric Wong}, {and} \bibinfo{person}{Zico Kolter}.} \bibinfo{year}{2020}\natexlab{}.
\newblock \showarticletitle{Overfitting in adversarially robust deep learning}. In \bibinfo{booktitle}{\emph{International Conference on Machine Learning}}. PMLR, \bibinfo{pages}{8093--8104}.
\newblock


\bibitem[Salman et~al\mbox{.}(2020)]%
        {salman2020adversarially}
\bibfield{author}{\bibinfo{person}{Hadi Salman}, \bibinfo{person}{Andrew Ilyas}, \bibinfo{person}{Logan Engstrom}, \bibinfo{person}{Ashish Kapoor}, {and} \bibinfo{person}{Aleksander Madry}.} \bibinfo{year}{2020}\natexlab{}.
\newblock \showarticletitle{Do adversarially robust imagenet models transfer better?}
\newblock \bibinfo{journal}{\emph{Advances in Neural Information Processing Systems}}  \bibinfo{volume}{33} (\bibinfo{year}{2020}), \bibinfo{pages}{3533--3545}.
\newblock


\bibitem[Shafahi et~al\mbox{.}(2019)]%
        {shafahi2019adversarial}
\bibfield{author}{\bibinfo{person}{Ali Shafahi}, \bibinfo{person}{Mahyar Najibi}, \bibinfo{person}{Mohammad~Amin Ghiasi}, \bibinfo{person}{Zheng Xu}, \bibinfo{person}{John Dickerson}, \bibinfo{person}{Christoph Studer}, \bibinfo{person}{Larry~S Davis}, \bibinfo{person}{Gavin Taylor}, {and} \bibinfo{person}{Tom Goldstein}.} \bibinfo{year}{2019}\natexlab{}.
\newblock \showarticletitle{Adversarial training for free!}. In \bibinfo{booktitle}{\emph{Advances in Neural Information Processing Systems}}. \bibinfo{pages}{3358--3369}.
\newblock


\bibitem[Song et~al\mbox{.}(2022)]%
        {song2022learning}
\bibfield{author}{\bibinfo{person}{Hwanjun Song}, \bibinfo{person}{Minseok Kim}, \bibinfo{person}{Dongmin Park}, \bibinfo{person}{Yooju Shin}, {and} \bibinfo{person}{Jae-Gil Lee}.} \bibinfo{year}{2022}\natexlab{}.
\newblock \showarticletitle{Learning from noisy labels with deep neural networks: A survey}.
\newblock \bibinfo{journal}{\emph{IEEE Transactions on Neural Networks and Learning Systems}} (\bibinfo{year}{2022}).
\newblock


\bibitem[Stutz et~al\mbox{.}(2020)]%
        {stutz2020confidence}
\bibfield{author}{\bibinfo{person}{David Stutz}, \bibinfo{person}{Matthias Hein}, {and} \bibinfo{person}{Bernt Schiele}.} \bibinfo{year}{2020}\natexlab{}.
\newblock \showarticletitle{Confidence-calibrated adversarial training: Generalizing to unseen attacks}. In \bibinfo{booktitle}{\emph{International Conference on Machine Learning}}. PMLR, \bibinfo{pages}{9155--9166}.
\newblock


\bibitem[Szegedy et~al\mbox{.}(2014)]%
        {szegedy2013intriguing}
\bibfield{author}{\bibinfo{person}{Christian Szegedy}, \bibinfo{person}{Wojciech Zaremba}, \bibinfo{person}{Ilya Sutskever}, \bibinfo{person}{Joan Bruna}, \bibinfo{person}{Dumitru Erhan}, \bibinfo{person}{Ian~J. Goodfellow}, {and} \bibinfo{person}{Rob Fergus}.} \bibinfo{year}{2014}\natexlab{}.
\newblock \showarticletitle{Intriguing properties of neural networks}. In \bibinfo{booktitle}{\emph{2nd International Conference on Learning Representations, {ICLR} 2014, Banff, AB, Canada, April 14-16, 2014, Conference Track Proceedings}}, \bibfield{editor}{\bibinfo{person}{Yoshua Bengio} {and} \bibinfo{person}{Yann LeCun}} (Eds.).
\newblock
\urldef\tempurl%
\url{http://arxiv.org/abs/1312.6199}
\showURL{%
\tempurl}


\bibitem[Tsipras et~al\mbox{.}(2020)]%
        {tsipras2020imagenet}
\bibfield{author}{\bibinfo{person}{Dimitris Tsipras}, \bibinfo{person}{Shibani Santurkar}, \bibinfo{person}{Logan Engstrom}, \bibinfo{person}{Andrew Ilyas}, {and} \bibinfo{person}{Aleksander Madry}.} \bibinfo{year}{2020}\natexlab{}.
\newblock \showarticletitle{From imagenet to image classification: Contextualizing progress on benchmarks}. In \bibinfo{booktitle}{\emph{International Conference on Machine Learning}}. PMLR, \bibinfo{pages}{9625--9635}.
\newblock


\bibitem[Wang et~al\mbox{.}(2023a)]%
        {wang2022yolov7}
\bibfield{author}{\bibinfo{person}{Chien-Yao Wang}, \bibinfo{person}{Alexey Bochkovskiy}, {and} \bibinfo{person}{Hong-Yuan~Mark Liao}.} \bibinfo{year}{2023}\natexlab{a}.
\newblock \showarticletitle{YOLOv7: Trainable bag-of-freebies sets new state-of-the-art for real-time object detectors}. In \bibinfo{booktitle}{\emph{Proceedings of the IEEE/CVF conference on computer vision and pattern recognition}}. \bibinfo{pages}{7464--7475}.
\newblock


\bibitem[Wang et~al\mbox{.}(2023b)]%
        {wang2023better}
\bibfield{author}{\bibinfo{person}{Zekai Wang}, \bibinfo{person}{Tianyu Pang}, \bibinfo{person}{Chao Du}, \bibinfo{person}{Min Lin}, \bibinfo{person}{Weiwei Liu}, {and} \bibinfo{person}{Shuicheng Yan}.} \bibinfo{year}{2023}\natexlab{b}.
\newblock \showarticletitle{Better diffusion models further improve adversarial training}. In \bibinfo{booktitle}{\emph{International conference on machine learning}}. PMLR, \bibinfo{pages}{36246--36263}.
\newblock


\bibitem[Wong et~al\mbox{.}(2020)]%
        {wong2020fast}
\bibfield{author}{\bibinfo{person}{Eric Wong}, \bibinfo{person}{Leslie Rice}, {and} \bibinfo{person}{J.~Zico Kolter}.} \bibinfo{year}{2020}\natexlab{}.
\newblock \showarticletitle{Fast is better than free: Revisiting adversarial training}. In \bibinfo{booktitle}{\emph{International Conference on Learning Representations}}.
\newblock
\urldef\tempurl%
\url{https://openreview.net/forum?id=BJx040EFvH}
\showURL{%
\tempurl}


\bibitem[Wu et~al\mbox{.}(2020)]%
        {wu2020adversarial}
\bibfield{author}{\bibinfo{person}{Dongxian Wu}, \bibinfo{person}{Shu-Tao Xia}, {and} \bibinfo{person}{Yisen Wang}.} \bibinfo{year}{2020}\natexlab{}.
\newblock \showarticletitle{Adversarial weight perturbation helps robust generalization}.
\newblock \bibinfo{journal}{\emph{Advances in Neural Information Processing Systems}}  \bibinfo{volume}{33} (\bibinfo{year}{2020}).
\newblock


\bibitem[Yao et~al\mbox{.}(2019)]%
        {yao2019latent}
\bibfield{author}{\bibinfo{person}{Yuanshun Yao}, \bibinfo{person}{Huiying Li}, \bibinfo{person}{Haitao Zheng}, {and} \bibinfo{person}{Ben~Y Zhao}.} \bibinfo{year}{2019}\natexlab{}.
\newblock \showarticletitle{Latent backdoor attacks on deep neural networks}. In \bibinfo{booktitle}{\emph{Proceedings of the 2019 ACM SIGSAC Conference on Computer and Communications Security}}. \bibinfo{pages}{2041--2055}.
\newblock


\bibitem[Yun et~al\mbox{.}(2021)]%
        {yun2021re}
\bibfield{author}{\bibinfo{person}{Sangdoo Yun}, \bibinfo{person}{Seong~Joon Oh}, \bibinfo{person}{Byeongho Heo}, \bibinfo{person}{Dongyoon Han}, \bibinfo{person}{Junsuk Choe}, {and} \bibinfo{person}{Sanghyuk Chun}.} \bibinfo{year}{2021}\natexlab{}.
\newblock \showarticletitle{Re-labeling imagenet: from single to multi-labels, from global to localized labels}. In \bibinfo{booktitle}{\emph{Proceedings of the IEEE/CVF conference on computer vision and pattern recognition}}. \bibinfo{pages}{2340--2350}.
\newblock


\bibitem[Zagoruyko and Komodakis(2016)]%
        {zagoruyko2016wide}
\bibfield{author}{\bibinfo{person}{Sergey Zagoruyko} {and} \bibinfo{person}{Nikos Komodakis}.} \bibinfo{year}{2016}\natexlab{}.
\newblock \showarticletitle{Wide Residual Networks}. In \bibinfo{booktitle}{\emph{Proceedings of the British Machine Vision Conference (BMVC)}}, \bibfield{editor}{\bibinfo{person}{Edwin R.~Hancock Richard C.~Wilson} {and} \bibinfo{person}{William A.~P. Smith}} (Eds.). \bibinfo{publisher}{BMVA Press}, Article \bibinfo{articleno}{87}, \bibinfo{numpages}{12}~pages.
\newblock
\showISBNx{1-901725-59-6}
\urldef\tempurl%
\url{https://doi.org/10.5244/C.30.87}
\showDOI{\tempurl}


\bibitem[Zhang et~al\mbox{.}(2019)]%
        {pmlr-v97-zhang19p}
\bibfield{author}{\bibinfo{person}{Hongyang Zhang}, \bibinfo{person}{Yaodong Yu}, \bibinfo{person}{Jiantao Jiao}, \bibinfo{person}{Eric Xing}, \bibinfo{person}{Laurent El~Ghaoui}, {and} \bibinfo{person}{Michael Jordan}.} \bibinfo{year}{2019}\natexlab{}.
\newblock \showarticletitle{Theoretically principled trade-off between robustness and accuracy}. In \bibinfo{booktitle}{\emph{International Conference on Machine Learning}}. PMLR, \bibinfo{pages}{7472--7482}.
\newblock


\bibitem[Zhang et~al\mbox{.}(2020b)]%
        {zhang2020attacks}
\bibfield{author}{\bibinfo{person}{Jingfeng Zhang}, \bibinfo{person}{Xilie Xu}, \bibinfo{person}{Bo Han}, \bibinfo{person}{Gang Niu}, \bibinfo{person}{Lizhen Cui}, \bibinfo{person}{Masashi Sugiyama}, {and} \bibinfo{person}{Mohan Kankanhalli}.} \bibinfo{year}{2020}\natexlab{b}.
\newblock \showarticletitle{Attacks which do not kill training make adversarial learning stronger}. In \bibinfo{booktitle}{\emph{International conference on machine learning}}. PMLR, \bibinfo{pages}{11278--11287}.
\newblock


\bibitem[Zhang et~al\mbox{.}(2018)]%
        {zhang2018deep}
\bibfield{author}{\bibinfo{person}{Lei Zhang}, \bibinfo{person}{Shuai Wang}, {and} \bibinfo{person}{Bing Liu}.} \bibinfo{year}{2018}\natexlab{}.
\newblock \showarticletitle{Deep learning for sentiment analysis: A survey}.
\newblock \bibinfo{journal}{\emph{Wiley Interdisciplinary Reviews: Data Mining and Knowledge Discovery}} \bibinfo{volume}{8}, \bibinfo{number}{4} (\bibinfo{year}{2018}), \bibinfo{pages}{e1253}.
\newblock


\bibitem[Zhang et~al\mbox{.}(2020a)]%
        {zhang2020towards}
\bibfield{author}{\bibinfo{person}{Xu-Yao Zhang}, \bibinfo{person}{Cheng-Lin Liu}, {and} \bibinfo{person}{Ching~Y Suen}.} \bibinfo{year}{2020}\natexlab{a}.
\newblock \showarticletitle{Towards robust pattern recognition: A review}.
\newblock \bibinfo{journal}{\emph{Proc. IEEE}} \bibinfo{volume}{108}, \bibinfo{number}{6} (\bibinfo{year}{2020}), \bibinfo{pages}{894--922}.
\newblock


\bibitem[Zhu et~al\mbox{.}(2021)]%
        {zhu2021reliable}
\bibfield{author}{\bibinfo{person}{Jianing Zhu}, \bibinfo{person}{Jiangchao Yao}, \bibinfo{person}{Bo Han}, \bibinfo{person}{Jingfeng Zhang}, \bibinfo{person}{Tongliang Liu}, \bibinfo{person}{Gang Niu}, \bibinfo{person}{Jingren Zhou}, \bibinfo{person}{Jianliang Xu}, {and} \bibinfo{person}{Hongxia Yang}.} \bibinfo{year}{2021}\natexlab{}.
\newblock \showarticletitle{Reliable adversarial distillation with unreliable teachers}.
\newblock \bibinfo{journal}{\emph{arXiv preprint arXiv:2106.04928}} (\bibinfo{year}{2021}).
\newblock


\bibitem[Zi et~al\mbox{.}(2021)]%
        {zi2021revisiting}
\bibfield{author}{\bibinfo{person}{Bojia Zi}, \bibinfo{person}{Shihao Zhao}, \bibinfo{person}{Xingjun Ma}, {and} \bibinfo{person}{Yu-Gang Jiang}.} \bibinfo{year}{2021}\natexlab{}.
\newblock \showarticletitle{Revisiting adversarial robustness distillation: Robust soft labels make student better}. In \bibinfo{booktitle}{\emph{Proceedings of the IEEE/CVF International Conference on Computer Vision}}. \bibinfo{pages}{16443--16452}.
\newblock


\bibitem[Zolfi et~al\mbox{.}(2021)]%
        {zolfi2021translucent}
\bibfield{author}{\bibinfo{person}{Alon Zolfi}, \bibinfo{person}{Moshe Kravchik}, \bibinfo{person}{Yuval Elovici}, {and} \bibinfo{person}{Asaf Shabtai}.} \bibinfo{year}{2021}\natexlab{}.
\newblock \showarticletitle{The translucent patch: A physical and universal attack on object detectors}. In \bibinfo{booktitle}{\emph{Proceedings of the IEEE/CVF Conference on Computer Vision and Pattern Recognition}}. \bibinfo{pages}{15232--15241}.
\newblock


\end{thebibliography}

\appendix

%----------------------------------------------------------------
\section{Analysis for Temperature Selection in Knowledge Distillation Frameworks}
\subsection{Analysis for the Ideal Case}
\label{sec:appendix_analysis}
\colorline{This analysis aims to theoretically explain why soft labels better represent ambiguous examples than one-hot labels under specific conditions and to illustrate the influence of temperature selection in the knowledge distillation framework.} To simplify the discussion, we make the following assumptions.
\begin{assumption}
Let $M$ be a model trained by the discrepancy in (\ref{eq:prob_SCE}) using soft labels $y^p$. When the training finishes, the probability $p$ of each input image $x$ output by $M$ equals the corresponding soft label $y^p$. 
\end{assumption}
This is a reasonable assumption because $M$ is well-trained,  implying that its discrepancy should be less than the given threshold. Eventually, the probability of each image output by the trained model can approximate the given soft label. To simplify the discussion, we just assume the threshold is negligible. 

Now we consider a two-class classification problem. For a given image $x$, its oracle soft label is ($y^p_1$, $y^p_2$). Without loss generality, we assume that $y^p_1 > y^p_2 \geq \xi$, where $\xi$ is a small positive number. Furthermore, let $(p_1^{T=t}, p_2^{T=t})$ be the output probability of $x$ normalized by softmax. Since the model is making the correct classification, $p_1^{T=t} > p_2^{T=t}$ for any temperature $t$.

There are two models, Model 1 and Model 2, which have the same network architecture but are trained by the soft labels that are generated from the teacher model with different temperatures $t_1$ and $t_2$ where $t_2 > t_1 \geq 1$. Using the definition in (\ref{eq:prob_SCE}), we can measure their difference of discrepancy: 
\begin{equation}
\label{eq:delta_L_1_to_T}
%    \Delta L^{t_1 \rightarrow t_2} = \mathop{\mathbb{E}}_{x \sim \mathcal{D}} \sum_{i=1}^k - y^p_i \log p_i^{T=t_2} - \mathop{\mathbb{E}}_{x \sim \mathcal{D}} \sum_{i=1}^k - y^p_i \log p_i^{T=t_1} \\
%    =  \mathop{\mathbb{E}}_{x \sim \mathcal{D}} - \sum_{i=1}^k y^p_i \log \left( \frac{p^{T=t_2}_i}{p^{T=t_1}_i} \right).
\begin{split}
    \Delta L^{t_1 \rightarrow t_2} &= \mathop{\mathbb{E}}_{x \sim \mathcal{D}} \sum_{i=1}^k - y^p_i \log p_i^{T=t_2} - \mathop{\mathbb{E}}_{x \sim \mathcal{D}} \sum_{i=1}^k - y^p_i \log p_i^{T=t_1} \\
    &=  \mathop{\mathbb{E}}_{x \sim \mathcal{D}} - \sum_{i=1}^k y^p_i \log \left( \frac{p^{T=t_2}_i}{p^{T=t_1}_i} \right).
\end{split}
\end{equation}
We say that the model trained by the soft label using temperature $t_2$ is better than temperature $t_1$ if
\begin{equation}
\label{eq:delta_cond}
    \Delta L^{t_1 \rightarrow t_2} < 0.
\end{equation}
The following theorem gives a sufficient condition of $ \Delta L^{t_1 \rightarrow t_2} < 0$. 

\begin{theorem}
\label{def:ambiguous}
With the above assumptions, if for each image, $p_1^{T=t} > p_2^{T=t}$ for any $t$, and
\begin{equation}
\label{eq:condition}
\frac{y^p_1}{y^p_2} \leq \left| \frac{\log(p^{T=t_2}_2/p^{T=t_1}_2)}{\log(p^{T=t_2}_1/p^{T=t_1}_1)} \right|
%\frac{y^p_1}{y^p_2} \leq \left| \log \left( \frac{p^{T=t_2}_2}{p^{T=t_1}_2} \right) \right| / \left| \log \left( \frac{p^{T=t_2}_1}{p^{T=t_1}_1} \right) \right|,
\end{equation}
then $\Delta L^{t_1 \rightarrow t_2} < 0$.
%there exists a temperature $\tau>1$ such that $\Delta L^{1 \rightarrow \tau} < 0$.
\end{theorem}
\begin{proof}
We start by simplifying $\Delta L^{t_1 \rightarrow t_2}$:
\begin{equation}
%    \Delta L^{t_1 \rightarrow t_2} = \mathop{\mathbb{E}}_{x \sim \mathcal{D}} -\sum_{i=1}^k y^p_i \frac{p^{T=t_2}_i}{p^{T=t_1}_i} \\
%    = \mathop{\mathbb{E}}_{x \sim \mathcal{D}} -y^p_1\log \left( \frac{p^{T=t_2}_1}{p^{T=t_1}_1} \right) -y^p_2\log \left( \frac{p^{T=t_2}_2}{p^{T=t_1}_2} \right).
\label{eq:delta_L_1_to_T_ambiguous}
\begin{split}
    \Delta L^{t_1 \rightarrow t_2} &= \mathop{\mathbb{E}}_{x \sim \mathcal{D}} -\sum_{i=1}^k y^p_i \frac{p^{T=t_2}_i}{p^{T=t_1}_i} \\
    &= \mathop{\mathbb{E}}_{x \sim \mathcal{D}} -y^p_1\log \left( \frac{p^{T=t_2}_1}{p^{T=t_1}_1} \right) -y^p_2\log \left( \frac{p^{T=t_2}_2}{p^{T=t_1}_2} \right).
\end{split}
\end{equation}
Because $p_1$ is decreasing as $T$ increases, and $p_2$ is increasing with $T$, $\log( p^{T=t_2}_1/p^{T=t_1}_1 )$ is negative and $\log( p^{T=t_2}_2/p^{T=t_1}_2 )$ is positive. 
If condition (\ref{eq:condition}) holds, 
for each image, we have 
\begin{equation}
\label{eq:ambiguous_cond}
\begin{split}
    \left| y^p_1\log \left( \frac{p^{T=t_2}_1}{p^{T=t_1}_1} \right) \right| &\leq  \left| y^p_2\log \left( \frac{p^{T=t_2}_2}{p^{T=t_1}_2} \right) \right| \\
    -y^p_1\log \left( \frac{p^{T=t_2}_1}{p^{T=t_1}_1} \right) & - y^p_2\log \left( \frac{p^{T=t_2}_2}{p^{T=t_1}_2} \right) \leq 0 \\
\end{split}
\end{equation}
Aggregating the results in (\ref{eq:ambiguous_cond}) by taking the expectation, we can show that $\Delta L^{t_1 \rightarrow t_2} < 0$.
%From the assumption $y^p_1>y^p_2$,  
%Because $p^{T=1}$ is closed to one-hot distribution, we know that $p^{T=\tau}_2 \gg p^{T=1}_2$ and the magnitudes of $p^{T=\tau}_1$ and $p^{T=1}_2$ are similar, we can find a temperature $\tau$, such that
%\begin{equation*}
%    1 \leq \left| \log \left( \frac{p^{T=\tau}_2}{p^{T=1}_2} \right) \right| / \left| \log \left( \frac{p^{T=\tau}_1}{p^{T=1}_1} \right) \right|.
%\end{equation*}
%Because each example in dataset satisfies the above condition, we ensure that $\Delta L^{1 \rightarrow \tau} < 0$.
\end{proof}

Generally, the probability distribution of $p^{T=1}$ is close to one-hot, so $p_1^{T=1}$ is close to 1, and $p_2^{T=1}$ is very small. As $T$ increases from 1 to $\tau$, $p_1^{T}$ decreases slowly and $p_2^{T}$ increases rapidly. Combing with the result of Theorem \ref{def:ambiguous}, we know that a low temperature is a good option. For example, we have two soft labels using $t_1=1$ and $t_2=\tau$, for $\tau > 1$. if $p^{T=1}=(0.999, 0.001)$ and $p^{T=\tau}=(0.99, 0.01)$, the right hand side of (\ref{eq:condition}) is about $254$. To make (\ref{eq:condition}) hold, we only need $y^p_2 > 0.004$, which is not a very strict condition for ambiguous examples. 

\colorline{We can also apply Theorem \ref{def:ambiguous} to explain that high temperature is not a proper choice. Suppose we have two soft labels from temperature $t_1$ and $t_2$, $1 \ll t_1 < t_2$.  When the temperature increases, the distribution becomes uniform gradually, as can be seen from (\ref{eqn:KD_softmax}).
Because the magnitudes of all probabilities are close, 
$p^{T=t_1}_1 \sim p^{T=t_2}_1$ and
$p^{T=t_1}_2 \sim p^{T=t_2}_2$,
the right hand side of (\ref{eq:condition}) is small. So making the condition hold becomes more difficult.
For example, for $p^{T=t_1}=(0.60, 0.40)$ and $p^{T=t_2}=(0.55, 0.45)$, right hand side of (\ref{eq:condition}) is about $1.35$. To make the condition hold, we need  $y^p_2 > 0.42$, which seems to be rare for real-world datasets.}

Of course, Theorem \ref{def:ambiguous} is a sufficient condition for $ \Delta L^{t_1 \rightarrow t_2} < 0$. In reality, if there are few images violating the condition (\ref{eq:condition}), the above arguments about the selection of temperature can still hold because the discrepancy is an expected value over all examples in the data set. 

\subsection{Simulation Analysis}
\label{sec:appendix_simulation}
\colorline{The discrepancy in (\ref{eq:prob_SCE}) is an expected value over all examples in the data set, meaning that some examples that violate the constraints are allowable. To further demonstrate the superiority of soft labels over one-hot vectors, we conducted experiments using synthetic data in the binary classification problem.}

We assume that three percent of the data are ambiguous, with oracle probabilities of $(1-\gamma, \gamma)$, where $0 < \gamma \leq 1$, and the rest of the data are one-hot vectors with probabilities of $(1-\epsilon, \epsilon)$, where $\epsilon$ is a tiny positive value to avoid the issue of computing cross-entropy with an infinite value. The probabilities output by the first model and the second model are $(1-\epsilon, \epsilon)$ and $(1-s, s)$ respectively for all data, where $s$ is a function of temperature. Simulation results reveal that the discrepancy of the second model is always smaller than that of the first model when $s = \gamma / 10$, which simulates soft labels obtained from low temperature. We can simulate the case for high temperatures by increasing the value of $s$. The results are getting worsen when the temperature raising since the distributions gradually become uniform. We obtain similar results when assuming $\gamma$ for each example is a random variable drawn from the interval $[0, 0.1]$.

On the other hand, for a k-class classification problem, we assume the oracle probabilities are drawn from Dirichlet distribution with a given concentration parameter $\alpha_{oracle}=\alpha$. Here, we assume that at least 97 percent of the data whose the largest probability greater than $0.995$. The probability obtained from the first model is a one-hot distribution, while the probability obtained from the second model follows another Dirichlet distribution with concentration parameter $\alpha_{soft} = \alpha_{oracle} \times c$, where $c$ is a positive value used to emulate soft labels generated by low temperature. When compared to the model trained with a one-hot distribution, the model trained with soft labels has a smaller discrepancy although it has a one percent misclassification rate. Furthermore, if the model trained with soft labels can classify all examples correctly and represent inter-class relations precisely, the discrepancy can be further reduced.

The results suggest that, despite the majority being well-defined examples, the quality of predictions is significantly influenced by ambiguous examples. The model trained with soft labels generated by low-temperature distillation can achieve better performance without encountering the gradient masking issue.

\section{White-box Evaluation with Transformer Architecture}
\label{sec:vit_exp}

\begin{table}[tb]
\centering
\caption{Robustness with transformer architecture on CIFAR-10 dataset}
\label{table:vit_cifar10}
\begin{tabular}{ c| c c c}
    paper & architecture            & acc\textsubscript{nat}[\%] & acc\textsubscript{AA}[\%]  \\
    \hline
    AWP + LTD                       & WRN-34-10 & 85.21 & 56.90 \\
    AWP                             & WRN-34-10 & 85.36 & 56.17 \\
    \hline
    TRADES +  LTD                   & WRN-34-10 & 85.63 & 55.09 \\
    TRADES \cite{pmlr-v97-zhang19p} & WRN-34-10 & 84.92 & 53.08 \\
    \hline
    AWP \cite{wu2020adversarial} + LTD & XCiT-S12-P8 & 76.94 & 40.91 \\
    AWP \cite{wu2020adversarial}       & XCiT-S12-P8 & 73.09 & 36.62 \\
\end{tabular}
\end{table}

\begin{table}[tb]
\centering
\caption{Robustness with transformer architecture on CIFAR-100 dataset}
\label{table:vit_cifar100}
\begin{tabular}{ c| c c c}
    paper & architecture            & acc\textsubscript{nat}[\%] & acc\textsubscript{AA}[\%]  \\
    \hline
    AWP + LTD                       & WRN-34-10 & 64.32 & 31.13 \\
    AWP                             & WRN-34-10 & 60.38 & 28.86 \\
    \hline
    AWP \cite{wu2020adversarial} + LTD & XCiT-S12-P8 & 54.35 & 20.00 \\
    AWP \cite{wu2020adversarial}       & XCiT-S12-P8 & 51.65 & 17.69 \\
\end{tabular}
\end{table}

\begin{table}[tb]
\centering
\caption{Robustness with transformer architecture on ImageNet dataset}
\label{table:vit_imagenet}
\begin{tabular}{ c| c c}
    architecture            & acc\textsubscript{nat}[\%] & acc\textsubscript{AA}[\%]  \\
    \hline
    SWIN-T-P4 + LTD                       & 70.78 & 35.66 \\
    SWIN-T-P4                             & 61.53 & 30.18 \\
\end{tabular}
\end{table}

Due to the high computational cost of full adversarial training, we were unable to conduct evaluations on the transformer architecture for all scenarios. For the CIFAR-10 and CIFAR-100 datasets, we utilized the \textit{xcit-small-24-p8} \cite{ali2021xcit} architecture, while for the ImageNet dataset, we employed the \textit{swin-tiny-patch4-window7-224} \cite{liu2021swin} architecture. For both CIFAR-10 and CIFAR-100, we used the AdamW optimizer with an initial learning rate of $0.0025$, whereas for the ImageNet dataset, the initial learning rate was set to $0.005$. The temperature fixed at $1.0$ for all dataset.

Since no prior adversarial training results for Transformer architectures on CIFAR-10 and CIFAR-100 are available in Robustbench \cite{croce2020robustbench}, and due to the substantial computational cost of training ImageNet models (approximately 45 days per full adversarial training cycle with 8 V100 GPUs), we were unable to reproduce similar works within the available time and GPU resources. Comparisons were made between models trained with one-hot labels and soft labels generated through LTD. Each experiment was conducted twice, and the best results are reported.

\colorline{The experimental results for CIFAR-10, CIFAR-100, and ImageNet are presented in Tables \ref{table:vit_cifar10}, \ref{table:vit_cifar100}, and \ref{table:vit_imagenet}, respectively. As can be seen, the robust accuracy improves by approximately 2\% across all datasets. These findings indicate that the proposed LTD method is effective when applied to transformer architectures, achieving enhanced robustness, particularly when the temperature is set to 1.} Due to limited computational budget, we did not conduct additional experiments with higher temperatures to assess whether robustness could be further improved.

However, we acknowledge that the natural accuracy of transformers is lower than that of convolutional models. These results are consistent with a recent work \cite{debenedetti2023light}. The primary reason for this is the need for more sophisticated data augmentation, learning rate decay policies, and a more thorough exploration of hyperparameters. We also observed that slight adjustments to the initial learning rate can significantly impact performance. To achieve better results, conducting a comprehensive ablation study and employing effective search algorithms are essential next steps. These directions will be explored in future work.

To examine whether synthetic training data could further enhance transformer performance, we conducted additional experiments incorporating data generated by diffusion models \cite{wang2023better}. Results are shown in Table \ref{table:vit_cifar10_ddpm}, where the term \textit{Extra} denotes the inclusion of synthetic data during training. With this augmentation, natural accuracy improved to approximately 79\%, and robust accuracy rose to 52\%. However, the use of LTD provided little additional benefit in this setting, indicating that the synthetic data itself was the primary contributor to improved robustness. The main trade-off, however, is a significant increase in training time — up to four to six times longer — which limits the feasibility of thorough hyperparameter optimization within constrained computational budgets.

\begin{table}[tb]
\centering
\caption{Robustness with transformer architecture including an additional data on CIFAR-10 dataset}
\label{table:vit_cifar10_ddpm}
\begin{tabular}{ c| c c c | c}
    paper & architecture            & acc\textsubscript{nat}[\%] & acc\textsubscript{AA}[\%] & Extra \\
    \hline
    AWP \cite{wu2020adversarial} + LTD & XCiT-S12-P8 & 79.85 & 52.10 & V\\
    AWP \cite{wu2020adversarial}       & XCiT-S12-P8 & 79.02 & 51.88 & V\\
    \hline
    AWP \cite{wu2020adversarial} + LTD & XCiT-S12-P8 & 76.94 & 40.91 & X\\
    AWP \cite{wu2020adversarial}       & XCiT-S12-P8 & 73.09 & 36.62 & X\\
\end{tabular}
\end{table}

\end{document}